\documentclass[11pt]{article}

\usepackage{times,blkarray,multirow,bbm,amsfonts,float,amsmath,amssymb,graphicx,url,amsthm}
\usepackage[noend]{algorithmic}
\usepackage[ruled]{algorithm2e}
\usepackage[utf8]{inputenc} 
\usepackage[T1]{fontenc}    
\usepackage{fullpage}
\usepackage{tikz}

\newtheorem{theorem}{Theorem}[section]
\newtheorem{claim}[theorem]{Claim}

\newtheorem{lemma}[theorem]{Lemma}

\newtheorem{remark}[theorem]{Remark}
\newtheorem{corollary}[theorem]{Corollary}
\newtheorem{definition}[theorem]{Definition}
\newtheorem{proposition}[theorem]{Proposition}
\newtheorem{example}[theorem]{Example}

\DeclareMathOperator*{\E}{\mathbb{E}}
\newcommand{\Ind}[1]{\mathbbm{1}\left[#1\right]}  
\DeclareMathOperator*{\argmin}{argmin}

\newcommand{\cH}{\mathcal{H}}
\newcommand{\cL}{\mathcal{L}}
\newcommand{\fpdisp}{\mathrm{FPR}}
\newcommand{\OPT}{\mathrm{OPT}}
\newcommand{\cD}{\mathcal{D}}
\newcommand{\X}{\mathcal{X}}    
\newcommand{\Y}{\mathcal{Y}}    
\renewcommand{\H}{\mathcal{H}}  
\newcommand{\convH}{\Delta(\H)}   
\newcommand{\D}{\mathcal{D}}    
\newcommand{\cP}{\mathcal{P}}    
\DeclareMathOperator*{\Prob}{\mathbb{P}} 
\newcommand{\plus}{\mathbf{+1}}  
\newcommand{\minus}{\mathbf{-1}}  
\newcommand{\plusa}{\mathbf{+a}}  
\newcommand{\minusa}{\mathbf{-a}}  

\DeclareMathOperator*{\Expectation}{\mathbb{E}}
\newcommand{\Ex}[2]{\Expectation_{#1}\left[#2\right]}

\newcommand{\Reg}{\mathrm{Reg}}
\newcommand{\RE}{\mathrm{RE}}
\newcommand{\FO}{\mathrm{FairCSC}}

\newcommand{\cgeneral}[1]{c^{(#1)}}
\newcommand{\cplus}{\cgeneral{+1}}
\newcommand{\cminus}{\cgeneral{-1}}

\title{Equal Opportunity in Online Classification with Partial Feedback}

\author{
Yahav Bechavod\thanks{School of Computer Science and Engineering, The Hebrew University. Email: \texttt{yahav.bechavod@cs.huji.ac.il}.}
\and 
Katrina Ligett\thanks{School of Computer Science and Engineering, The Hebrew University. Email: \texttt{katrina@cs.huji.ac.il}.}
\and
Aaron Roth\thanks{Department of Computer and Information Science, University of Pennsylvania. Email: \texttt{aaroth@cis.upenn.edu}.}
\and
Bo Waggoner\thanks{Department of Computer Science, University of Colorado. Email: \texttt{bwag@colorado.edu}.}
\and
Zhiwei Steven Wu\thanks{Computer Science and Engineering Department, University of Minnesota. Email: \texttt{zsw@umn.edu}.}
}

\begin{document}

\maketitle

\begin{abstract}
  We study an online classification problem with partial feedback in which individuals arrive one at a time from a fixed but unknown distribution, and must be classified as positive or negative. Our algorithm only observes the true label of an individual if they are given a positive classification. This setting captures many classification problems for which fairness is a concern: for example, in criminal recidivism prediction, recidivism is only observed if the inmate is released; in lending applications, loan repayment is only observed if the loan is granted. We require that our algorithms satisfy common statistical fairness constraints (such as equalizing false positive or negative rates --- introduced as ``equal opportunity'' in \cite{HPS16}) \emph{at every round}, with respect to the underlying distribution. We give upper and lower bounds characterizing the cost of this constraint in terms of the regret rate (and show that it is mild), and give an \emph{oracle efficient} algorithm that achieves the upper bound.
\end{abstract}

\newpage

\section{Introduction}

Many real-world prediction tasks in which fairness concerns arise --- such as online advertising, short-term hiring, lending micro-loans, and predictive policing --- are naturally modeled as online binary classification problems, but with an important twist: feedback is only received for one of the two classification outcomes. Clickthrough is only observable if the advertisement was shown in the first place; worker performance is only observed for candidates who were actually hired; those who are denied a loan never have an opportunity to demonstrate that they would have repaid; only if police troops were dispatched to a precinct are they able to detect unreported crimes. Applying standard techniques for enforcing statistical fairness constraints on the \emph{gathered data} can thus lead to pernicious feedback loops that can lead to classifiers that badly violate these constraints on the underlying distribution. This kind of failure to ``explore'' has been highlighted as an important source of algorithmic unfairness --- for example, in predictive policing settings \cite{LI16,ensign2018runaway,ensign-apples}.

To avoid this problem, it is important to explicitly manage the exploration/exploitation tradeoff that characterizes learning in partial feedback settings, which is what we study in this paper. We ask for algorithms that enforce well-studied statistical fairness constraints across two protected populations (we focus on the ``equal opportunity'' constraint of \cite{HPS16}, which enforces equalized false positive rates or false negative rates, but our techniques also apply to other statistical fairness constraints like ``statistical parity'' \cite{awareness}). In particular, we ask for algorithms that satisfy these constraints (with respect to the unknown underlying distribution) at every round of the learning procedure. The result is that the fairness constraints restrict how our algorithms can \emph{explore}, not just how they can exploit, which makes the problem of fairness-constrained online learning substantially different from in the batch setting. The main question that we explore in this paper is: ``\emph{how much does the constraint of fairness impact the regret bound of learning algorithms?}''

\subsection{Our Model and Results}
In our setting, there is an unknown distribution $\D$ over \emph{examples}, which are triples $(\hat x,a,y) \in \X \times \{-1,1\} \times \{-1,1\}$.
Here $\hat x \in \X$ represents a vector of features in some arbitrary \emph{feature space}, $a \in A = \{\pm 1\}$ is the \emph{group} to which this example belongs (which we also call the \emph{sensitive feature}), and $y \in \Y=\{\pm 1\}$ is a binary label. We write $x$ to denote a pair $(\hat x,a)$ -- the set of all features (including the sensitive one) that the learner has access to.

In each round $t\in [T]$, our learner selects hypotheses from a \emph{hypothesis class} $\H$ consisting of functions $h: \X \times A \to \Y$ recommending an action (or \emph{label}) as a function of the features (potentially including the sensitive feature). We take the positive label to be the one that corresponds to observing feedback (hiring a worker, admitting a student, approving a loan, releasing an inmate, etc.) We allow algorithms that randomize over $\H$. Let $\convH$ be the set of probability distributions over $\H$. We refer to a $\pi \in \convH$ as a \emph{convex combination of classifiers}.

\begin{definition}[False positive rate]
  For a fixed distribution $\D$ on examples, we define the false positive rate (FPR)
  of a convex combination of classifiers $\pi \in \convH$ on group $j \in \{\pm 1\}$ to be
  \begin{align*}
    FPR_j(\pi) &= \Prob(\pi(x)=+1|a=j,y=-1) = \E_{h \sim \pi}\left[\Prob_{(x,y)\sim \D}(h(x)=+1|a=j,y=-1)\right].
  \end{align*}
\end{definition}

We denote the \emph{difference} between false positive
rates between populations as
\begin{align*}
\Delta_{FPR}(\pi) :&= FPR_{1}(\pi) - FPR_{-1}(\pi).
\end{align*}
The fairness constraint we impose on our classifiers in this paper asks that false positive rates be approximately equalized across populations \emph{at every round} $t$. Throughout, analogous results hold for false negative rates. These constraints were called \emph{equal opportunity} constraints in \cite{HPS16}.
\begin{definition}[$\gamma$-equalized rates \cite{HPS16}]
  Fix a distribution $\D$.
  A convex combination $\pi \in \convH$ satisfies the \emph{$\gamma$-equalized false positive rate} (\emph{$\gamma$-EFP}) constraint if $|\Delta_{FPR}(\pi)| \leq \gamma$.
  The   \emph{$\gamma$-equalized false negative rate ($\gamma$-EFN)} constraint is defined analogously.

 We informally use the term \emph{$\gamma$-fair} to refer to such a classifier or combination of classifiers.
\end{definition}

As we will see in Definition~\ref{def:gammaEFPdelta}, we will actually
allow our algorithm to have a tiny probability of ever breaking the
fairness constraint.

\begin{remark}
The sources of unfairness we deal with here are the differential abilities of models in $\cH$ to predict on different populations (which we inherit from the batch setting), and the biased data collection inherent in online partial information settings. We use equal opportunity constraints only as a canonical example of a statistical fairness constraint and do not take the position that it is always the right one. Our techniques also apply to other constraints like statistical parity.
\end{remark}

Note that the fairness constraint is defined with respect to the true underlying distribution $\D$.
One of the primary difficulties we face is that in early rounds, the learner has very little information about $\D$, and yet is required to satisfy the fairness constraint with respect to $\D$.

It is straightforward to see (and a consequence of a more general lower bound that we prove) that a $\gamma$-fair algorithm cannot in general achieve non-trivial regret to the set of $\gamma$-fair convex combinations of classifiers, because of ongoing statistical uncertainty about the fairness level for all non-trivial classifiers. Thus our goal is to minimize our regret to the $\gamma$-fair convex combination of classifiers that has the lowest classification error on $\D$, while guaranteeing that our algorithm only deploys convex combinations of classifiers that guarantee fairness level $\gamma'$ for some $\gamma' > \gamma$. Clearly, the optimal regret bound will be a function of the gap $(\gamma'-\gamma)$, and one of our aims is to characterize this tradeoff.

\paragraph{An initial approach.}
Even absent fairness constraints, the problem of learning from partial feedback is challenging and has been studied under the name ``apple tasting''\cite{appletasting}.
Via standard techniques, it can be reduced to a contextual bandit problem \cite{BGPS10}. Therefore, an initial approach starts with the observation (Lemma~\ref{support_size}) that although the set of ``fair distributions over classifiers'' is continuously large, the ``fair empirical risk minimization'' problem only has a single constraint, and so we may without loss of generality consider distributions over hypotheses $\H$ that have support of size 2. By an appropriate discretization, this allows us to restrict attention to a finite net of classifiers whenever $\H$ itself is finite. From this observation, one could employ a simple strategy to obtain an information theoretic result: pair an ``exploration'' round in which all examples are classified as positive (so as to gather label information and rule out classifiers that substantially violate the fairness constraints on the gathered data), with an ``exploitation'' round in which a generic contextual bandits algorithm like a variant of EXP4 \cite{EXP4,freedman} is run over the surviving finite (but exponentially large) number of empirically fair distributions from the net. This simple approach yields the following bound: For any parameter $\alpha \in [0.25, 0.5]$, there is an algorithm that obtains a regret bound of $O(T^{2\alpha})$ to the best $\gamma$ fair classifier  while satisfying a $\gamma'$-fairness constraint at every round with a gap of $(\gamma'-\gamma) = O(T^{-\alpha})$.

\paragraph{Our results.}
We show that the tradeoff achieved by the inefficient algorithm is tight by proving a lower bound in Section \ref{lower_res}.
In some sense, the computational inefficiency of the simple bandits reduction above is unavoidable, because we measure the regret of our learner with respect to 0/1 classification error, which is computationally hard to minimize, even for very simple classes $\H$ (see, e.g., \cite{hard1,hard2,hard3}). However, we can still hope to give an \emph{oracle efficient algorithm} for our problem. This approach, which is common in the contextual bandits literature, assumes access to an ``oracle'' which can in polynomial time solve the empirical risk minimization problem over $\H$ (absent fairness constraints), and is an attractive way to isolate the ``hard part'' of the problem that is often tractable in practice. Our main result, to which we devote the body of the paper, is to show that access to such an oracle is sufficient to give a polynomial-time algorithm for the fairness-constrained learning problem, matching the simple information theoretically optimal bounds described above. To do this, we use two tools. Our high-level strategy is to apply the oracle efficient stochastic contextual bandit algorithm from \cite{minimonster}. In order to do this, we need to supply it with an offline learning oracle for the set of classifiers that can with high probability be certified to satisfy our fairness constraints given the data so far. We construct an approximate oracle for this problem (given a learning oracle for $\H$) using the oracle-efficient reduction for offline fair classification from \cite{MSR}. We need to overcome a number of technical difficulties stemming from the fact that the fair oracle that we can construct is only an approximate empirical risk minimizer, whereas the oracle assumed in \cite{minimonster} is exact. Moreover, the algorithm from \cite{minimonster} assumes a finite hypothesis class, whereas we need to obtain no regret to a continuous family of distributions over hypotheses. The final result is an oracle-efficient algorithm trading off between regret and fairness, allowing for a regret bound of $O(T^{2\alpha})$ to the best $\gamma$-fair classifier while satisfying $\gamma'$-fairness at every round, with a gap of $\gamma'-\gamma = O(T^{-\alpha})$ for $\alpha \in [0.25,0.5]$.

\begin{figure}[t]
\centering
\begin{tikzpicture}
\fill[fill=red,opacity=0.3]
    (3,3.75)
 -- (1,3.65909)
 -- (1,1.5)
 -- (2,2);

\fill[fill=red,opacity=0.3]
    (-1,3.52272)
 -- (-1,1.28571)
 -- (-1.75,1.5)
 -- (-1.5,3.5);

\fill[fill=green,opacity=0.3]
	(1,1.5)
 -- (1,3.65909)
 -- (-1,3.52272)
 -- (-1,1.28571)
 -- (0,1);

\draw[help lines, color=gray!30, dashed] (-4.9,0) grid (4.9,4.9);
\draw[->,ultra thick] (-5,0)--(5,0) node[right]{$\Delta_{FPR}(\pi)$};
\draw[->,ultra thick] (0,0)--(0,5) node[above]{$\Prob(\pi(x,a)=y)$};
\draw[gray, dashed] (-1,0) -- (-1,5);
\draw[gray, dashed] (1,0) -- (1,5);
\filldraw[black] (-1,0) circle (2pt) node[anchor=north] {$-\gamma$};
\filldraw[black] (1,0) circle (2pt) node[anchor=north] {$\gamma$};
\filldraw[black] (0,1) circle (2pt) node[anchor=east] {$pos$};
\filldraw[black] (0,3) circle (2pt) node[anchor=east] {$neg$};
\filldraw[black] (3,3.75) circle (2pt) node[anchor=west] {$f_1$};
\filldraw[black] (-1.5,3.5) circle (2pt) node[anchor=east] {$f_2$};
\filldraw[black] (-1.75,1.5) circle (2pt) node[anchor=north] {$f_3$};
\filldraw[black] (2,2) circle (2pt) node[anchor=west] {$f_4$};

\draw[gray, thick] (3,3.75) -- (-1.5,3.5);
\draw[gray, thick] (-1.5,3.5) -- (-1.75,1.5);
\draw[gray, thick] (-1.75,1.5) -- (0,1);
\draw[gray, thick] (0,1) -- (2,2);
\draw[gray, thick] (2,2) -- (3,3.75);

\filldraw[black] (1,3.65909) circle (2pt) node[anchor=south] {$OPT$};
\end{tikzpicture}
\caption{An illustration of the feasible accuracy, fairness levels using the set of policies $\convH$ for class $\H = \{pos,neg,f_1,f_2,f_3,f_4\}$. The green+red regions mark all feasible accuracy, fairness levels for convex combinations ($\convH$). The green region marks all feasible accuracy, fairness levels for $\gamma$-EFP convex combinations in $\convH$. $OPT$ marks the optimal $\gamma$-EFP policy in $\convH$. $pos$, $neg$ stand for the constant $+1, -1$ classifiers, respectively.} \label{fig:Convex}
\end{figure}
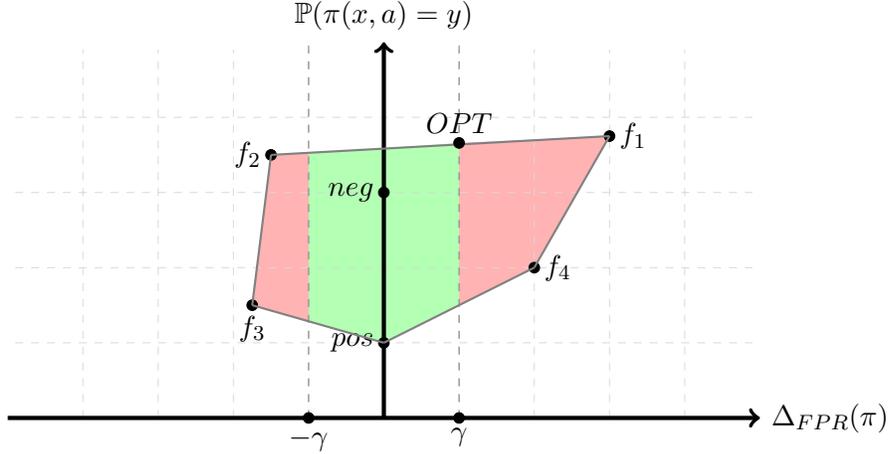

\subsection{Additional Related Work}
We build upon two lines of work in the fair machine learning literature. Much of this literature studies batch classification problems under a variety of statistical fairness constraints, each approximately equalizing a statistic of interest across protected sub-populations: raw classification rates \cite{CV10,KAS11,feldman2015certifying} (\emph{statistical parity} \cite{awareness}), positive predictive value \cite{KMR16,Chou17}, and false positive and false negative rates~\cite{KMR16,Chou17,HPS16} ; see \cite{survey} for more examples. One of the attractions of this family of constraints is that they can generally be enforced in the batch setting without assumptions about the data distribution.

There is also a literature on fair online classification and regression in the contextual bandit setting \cite{JKMR16,JKMNR18,liu2017calibrated}. These papers have studied the achievable regret when the learning algorithm must satisfy a fairness constraint \emph{at every round}, as we require in this paper. However, previous work has demanded stringent \emph{individual} fairness constraints that bind on particular pairs of individuals, rather than just the average behavior of the classifier over large groups (as statistical fairness constraints do). As a result, strong realizability assumptions had to be made in order to derive non-trivial regret bounds (and even in the realizable setting, simple concept classes like conjunctions were shown to necessitate an exponentially slower learning rate when paired with individual fairness constraints \cite{JKMR16}). Our paper interpolates between these two literatures: we ask for \emph{statistical} fairness constraints to be enforced at every round of a learning procedure, and show that in this case, even without any assumptions at all, the effect of the fairness constraint on the achievable regret bound is mild.

Recently, \cite{BGLS18} considered the problem of enforcing statistical fairness in a full information online learning setting, but from a very different perspective. They showed that in the \emph{adversarial} setting, it can be impossible to satisfy equalized false positive and false negative rate constraints averaged over history, even when the adversary is constrained so that each individual classifier in the hypothesis class individually satisfies the constraint. In contrast, they show that it is possible to do this for the equalized error rates constraint. Our setting is quite different: on the one hand, we require that our algorithm satisfy its fairness constraint \emph{at every round}, not just on average over the history, and we work in a partial information setting. On the other hand, we assume that examples are drawn from a distribution, rather than being adversarially chosen.

\section{Additional Preliminaries} \label{sec:definitions}
Throughout the paper, we assume $\plus, \minus \in \H$, where $\plus$ and $\minus$ are the two \emph{constant} classifiers (that is, $\plus(x) = 1$ and $\minus(x) = -1$ for all $x$). In some cases, we will additionally assume $\plusa, \minusa \in \H$, where $\plusa$ and $\minusa$ are the identity function (and its negation) on the sensitive feature (that is, $\plusa(\hat x,a) = a$ and $\minusa(\hat x,a) = -a$ for all $\hat x,a$).

\paragraph{The Online Setting:}
The learner interacts with the environment as follows.\\

\noindent\fbox{%
    \parbox{\textwidth}{%
\centerline{Online Learning in Our Partial Feedback Setting}
\begin{algorithmic}
\FOR {$t=1,...,T$}
	\STATE Learner chooses a convex combination $\pi_t \in \convH$.
	\STATE Environment draws $(x_t,y_t) \sim \mathcal{D}$ independently; learner observes $x_t$.
	\STATE Learner labels the point $\hat{y}_t = h_t(x_t)$, where $h_t \sim \pi_t$.
        \IF{$\hat{y}_t = +1$} \STATE Learner observes $y_t$.
        \ENDIF
\ENDFOR
\end{algorithmic}
    }%
}\\~\\

We illustrate our setting with an example.
\begin{example}
  A lender must accept (label $+1$) or reject (label $-1$) loan applications from individuals who may be either female (group $F$) or male (group $M$).
  After accepted applicants receive a loan, their true quality is observed: $+1$ if they repay the loan, and $-1$ if they do not repay.
  Rejected applicants' true quality is unobserved.
  Suppose every application's group and true quality are drawn independently and uniformly from $\{F,M\} \times \{\pm 1\}$.
  If $h_1$ labels all females' submissions $+1$ and all males' $-1$, then $FPR_F(h_1) = 1$, $FNR_F(h_1) = 0$, $FPR_M(h_1) = 0$, and $FNR_M(h_1) = 1$.
  If $h_2(x) = -h_1(x)$ (accepts if and only if male) and $\pi$ mixes uniformly between $h_1$ and $h_2$, then $\pi$ is equivalent to a random coin toss, and satisfies both $0-EFP$ and $0-EFN$.
\end{example}

\paragraph{Regret:}
We measure a learner's performance using $0$-$1$ loss, $\ell(\hat{y}_t,y_t) = \Ind{\hat{y}_t \neq y_t}$.
Given a class of distributions $\cP$ over $H \subseteq \H$ and a sequence of $T$ examples, the optimal convex combination of hypotheses from $H$ in hindsight is defined as
$\pi^*(\cP) = \argmin_{\pi \in \cP} \sum_{t=1}^T \Expectation_{h \sim \cP}[\ell(h(x_t), y_t)]$.

A learner's (pseudo)-regret  with respect to $\cP$ is
  \[ \text{Regret} = \sum_{t=1}^T \Expectation_{(x_t,y_t) \sim \D} [\ell(h(x_t), y_t)] ~ - ~ \sum_{t=1}^T \Expectation_{(x_t,y_t) \sim \D \ h \sim \pi^*(\cP)} [\ell(h(x_t), y_t)] . \]
In particular, when $\cP = \{\pi \in \Delta(\H) : \text{$\pi$ satisfies $\gamma$-EFP}\}$, we call this the learner's $\gamma$-EFP regret.

Finally, we ask for online learning algorithms that satisfy the following notion of fairness:
\begin{definition}[$\gamma$-EFP($\delta$) online learning algorithm]\label{def:gammaEFPdelta}
An online learning algorithm is said to satisfy $\gamma$-EFP($\delta$) fairness (for $\delta \in [0,1]$) if, with probability $1-\delta$ over the draw of $\{(x_t,a_t,y_t)\}_{t=1}^T \sim \D^T$, simultaneously for all rounds $t \in [T]$: $\pi_t$ satisfies $\gamma$-EFP.
\end{definition}

\paragraph{Cost Sensitive Classification Algorithms:}
We aim to give oracle-efficient online learning algorithms --- that
is, algorithms that run in polynomial time per round, assuming access
to an oracle which can solve the corresponding offline empirical risk
minimization problem. Concretely, we assume oracles for solving
\emph{cost sensitive classification (CSC)} problems over $\H$, which
are defined by a set of examples $x_j$ and a set of weights
$c_j^{-1}, c_j^{+1} \in \mathbb{R}$ corresponding to the cost of a
negative and positive classification respectively.
\begin{definition}
Given an instance of a CSC problem $S = \{x_j, c_j^{-1}, c_j^{+1}\}_{j=1}^n$, a \emph{CSC oracle} $\mathcal{O}$ for $\H$ returns $\mathcal{O}(S) \in \arg\min_{h \in \H} \;  \sum_{j=1}^n \cgeneral{h(x_j)}_j$. From these oracles, we will  construct $\nu$-approximate CSC oracles that may have restricted ranges $\Pi \subseteq \Delta(\H)$. Such oracles return $\mathcal{O}_{\nu}(S) = \pi \in \Pi$ such that $\Expectation_{h \sim \pi}[\sum_{j=1}^n \cgeneral{h(x_j)}_j] \leq  \arg\min_{\pi \in \Pi} \;  \Expectation_{h \sim \pi}[\sum_{j=1}^n \cgeneral{h(x_j)}_j] + \nu$.
\end{definition}

\paragraph{From ``Apple Tasting'' to Contextual Bandits:}
Online classification problems under the feedback model we study were first described as ``Apple Tasting'' problems \cite{appletasting}. The algorithm's loss at each round accumulates according to the following loss matrix: $$L =
\begin{blockarray}{ccc}
 & y = +1 & y = -1 \\
\begin{block}{c(cc)}
  \hat y = +1 & 0 & 1 \\
  \hat y = -1 & 1 & 0  \\
\end{block}
\end{blockarray},$$ but feedback is only observed for positive
classifications (when $\hat y = +1$). This is a different feedback
model than the more commonly studied \emph{contextual bandits}
setting. In that setting, the learner always gets to observe the loss
of the selected action (regardless of a positive or a negative
classification). We will defer the formal description of contextual
bandits to the appendix.

It is nevertheless straightforward to transform the apple tasting setting into the contextual bandits setting (similar observations have been previously made \cite{BGPS10}).

\begin{proposition} \label{prop:reduce-to-bandits} Let $\mathcal{A}$
  be an contextual bandits algorithm that guarantees a regret bound
  $R(T)$ with probability $1-\delta$. There exists a
  transformation that maps feedbacks for apple tasting to feedbacks
  for contextual bandits such that $\mathcal{A}$ gurantees regret
  bound $2 R(T)$ with probability $1-\delta$ when running
  on the transformed feedbacks on any apple tasting instance.
\end{proposition}

\paragraph{Baseline approaches.}
Given the reduction above, we can draw on standard methods from
contextual bandits to solve our fair online learning problem.  A
simple baseline approach that is oracle-efficient is to perform
``exploration-then-exploitation'': the learner first ``explores'' by
predicting $+1$ for roughly $T^{2/3}$ rounds, then ``exploits''
what we have learned by deploying the (empirically) best performing
fair policy. This approach would guarantee a sub-optimal regret bound
of $\tilde O(T^{\frac{2}{3}})$ to the best $\gamma$-fair classifier,
while satisfying a $\gamma'$-fairness constraint at every round with a
gap of $(\gamma'-\gamma) = O(T^{-\frac{1}{3}})$.

A more sophisticated approach starts with the observation
(Lemma~\ref{support_size}) that although the set of ``fair
distributions over classifiers'' is continuously large, the ``fair
empirical risk minimization'' problem only has a single constraint,
and so we may without loss of generality consider distributions over
hypotheses $\H$ that have support of size 2. By an appropriate
discretization, this allows us to restrict attention to a finite net
of classifiers whenever $\H$ itself is finite. From this observation,
one could employ a simple strategy to obtain an information theoretic
result: Fix any parameter $\alpha\in[1/4,1/2]$. The learner first
predicts $+1$ for roughly $T^{2\alpha}$ rounds, then uses the
collected data to define a set of fair policies according to the
observed empirical distribution, and lastly runs the EXP4
algorithm~\cite{EXP4,freedman} over the set of fair policies. Such
algorithm obtains a regret bound of $O(T^{2\alpha})$ to the best
$\gamma$-fair classifier, while satisfying a $\gamma'$-fairness
constraint at every round with a gap of
$(\gamma'-\gamma) = O(T^{-\alpha})$. However, this algorithm needs to
maintain a distribution of exponential size, and our goal is to match
its regret rate with an oracle-efficient algorithm.

\section{Oracle-Efficient Algorithm} \label{sec:comp}
Our algorithm proceeds in two phases. First, during the first $T_0$
rounds, the algorithm performs \emph{pure exploration} and always
predicts $+1$ to collect labelled data. Because constant classifiers
exactly equalize the false positive rates across populations, each
exploration round satisfies our fairness constraint. The algorithm
then uses the collected data to form empirical fairness constraints,
which we use to define our construction of a fair CSC oracle, given a
CSC oracle unconstrained by fairness. Then, in the remaining rounds,
we will run an adaptive contextual bandit algorithm that minimizes
cumulative regret, while satisfying the empirical fairness constraint
at every round.

We make two mild assumptions to simplify our analysis and the statement of our final bounds. First, we assume that 
negative examples from each of the two protected groups have constant probability mass:
  $\Pr[a= 1, y=-1], \Pr[a = -1, y=-1] \in \Omega(1)$. Second, we assume that the hypothesis class $\H$ contains the two constant classifiers and the identity function and its negation on the protected attribute: $\{\plus,\minus,\plusa,\minusa\} \subseteq \H$.

Our main theorem is as follows:

\begin{theorem}\label{upper}
  For any $\H$ and data distribution satisfying the two mild
  assumptions above, there exists an oracle-efficient algorithm that
  takes parameters $\delta \in [0,\frac{1}{\sqrt{T}}]$ and
  $\gamma \geq 0$ as input and satisfies $(\gamma + \beta)$-EFP($\delta$)
  fairness and has an expected regret at most
  $\tilde O(\sqrt{T} \ln(|\H|/\delta) )$ with respect to the class of
  $\gamma$-EFP fair policies, where
  $\beta =  O(\sqrt{\ln(|\H|/\delta)}/T^{1/4})$.
\end{theorem}

\begin{remark}
More generally, we can extend Theorem~\ref{upper} to give an algorithm that satisfies $(\gamma + \beta)$-EFP($\delta$) for any $\beta > 0$, and achieves an expected regret at most $\tilde O\left(\frac{\ln(\frac{|\H|}{\delta})}{\beta^2} + \sqrt{T} \ln(|\H|/\delta) \right)$ with respect to the class of $\gamma$-EFP fair policies. 
\end{remark}

\begin{remark}
We state our theorem in what we believe is the most attractive parametric regime: when it can obtain a regret bound of $O(\sqrt{T})$. But it is straightforward, by modifying the length of the exploration round, to obtain a more general tradeoff---a regret bound of $O(T^{2\alpha})$ with respect to the set of $\gamma$-EFP fair policies, while satisfying $(\gamma + O(T^{-\alpha}))$-EFP($\delta$) fairness, for any $\alpha \in [1/4,1/2]$. This tradeoff is tight, as we show in Section \ref{lower_res}.
\end{remark}

\paragraph{Algorithm.}
The outline of our algorithm is as follows.
\begin{enumerate}
    \item Label the first $T_0$ arrivals as $\hat{y}_t = 1$; observe their true labels.
    \item  Based on this data, construct an efficient \emph{FairCSC oracle}.
        The oracle will be given a cost-sensitive classification objective.
        It returns an approximately-optimal convex combination $\pi$ of hypotheses subject to the linear constraint of $(\gamma + T^{-1/4})$-EFP on the empirical distribution of data.
        We show the algorithm can be implemented to always return a member of $\Pi$, defined to be the set of mixtures on $\H$ with support size two whose empirical fairness on the exploration data is at most $\gamma + \tilde{O}(T^{-1/4})$.
    \item Instantiate a bandit algorithm with policy class $\Pi$.
        The bandit algorithm, a modification of \cite{minimonster}, is described in detail in the next sections.
        In order to select its hypotheses, the bandit algorithm makes calls to the FairCSC oracle we implemented above.
    \item For the remaining rounds $t > T_0$, choose labels $\hat{y}_t$ selected by the bandit algorithm and provide feedback to the bandit algorithm via the reduction given by Proposition \ref{prop:reduce-to-bandits}.
\end{enumerate}

\paragraph{Analysis.}
In the remainder of this section, we present our analysis in three
main steps: 
\begin{itemize}
    \item First, we study the empirical fairness constraint given
by the data collected during the exploration phase and give a
reduction from a cost-sensitive classification problem subject to such
fairness constraint to a standard cost-sensitive classification
problem absent the constraint, based on \cite{MSR}. We need to perform
two modifications on the reduction method in~\cite{MSR}. First, we
allow our algorithm to handle fairness constraints defined by a
separate data set that is different from the one defining the cost
objective.  Second, we also provide a fair approximate CSC oracle
that returns a sparse solution, a distribution over $\H$ with support
size of at most 2. This will be useful for establishing uniform
convergence.

\item Next, we present the algorithm run in the second phase: at each round
 $t > T_0$, the algorithm makes a prediction based on a randomized
 policy $\pi_t\in \Delta(\H)$, which is a solution to a feasibility
 program given by \cite{minimonster}. We show how to rely on an approximate fair CSC
 oracle to solve this program efficiently. Consequently, we
 generalize the results of \cite{minimonster} to the setting in which the given oracle
 may only optimize the cost sensitive objective approximately. This may be of independent interest.
 
\item Finally, we bound the deviation between the algorithm's empirical regret and
 true expected regret. This in particular requires uniform
 convergence over the entire class of fair randomized policies, which
 we show by leveraging the sparsity of the fair distributions.
\end{itemize}

We now give the proof of Theorem \ref{upper}, with forward references to needed theorems and lemmas.
\begin{proof}[Proof of Theorem \ref{upper}]
  We set $T_0 = \Theta(\sqrt{T \ln(|\H|/\delta)})$. First,
  Lemma~\ref{support_size1} shows that given our empirical EFP
  constraint, there exists an optimal policy of support size at most
  2. Next, Lemma \ref{lemma:oracle-efp} shows that, with probability
  $1-\delta$ over arrivals $1,\dots,T_0$, all convex combinations
  $\pi \in \Pi$ satisfy $\hat{\gamma}$-EFP for
  $\hat{\gamma} = \gamma + \beta$,
  $\beta = O\left(\sqrt{\ln(|\H|/\delta)}/T^{1/4}\right)$.
  It also implies that
  the optimal $\gamma$-fair policy is in the class.  Theorem
  \ref{reduction} shows that, given a CSC oracle for $\H$, we can
  implement an efficient approximate CSC oracle for this class $\Pi$.
  Theorem \ref{thm:bandit-regret} shows that, given an approximate CSC
  oracle for any class, there is an efficient bandit algorithm that
  deploys policies from this class and achieves expected regret
  $O\left(\ln\left(|\H|T/\delta\right) \sqrt{T}\right)$.

  \textbf{Fairness:} In the first $T_0$ rounds we deploy $\mathbf{+1}$ which is $0$-fair, and in the remaining rounds we deploy policies only from $\Pi$.
  With probability $1-\delta$ over the exploration data, every member of $\Pi$ is $(\gamma+\beta)$-fair.

  \textbf{Regret:} The algorithm's regret is at most $T_0$ plus its regret, on rounds $T_0+1,\dots,T$, to the optimal policy in $\Pi$.
  By Proposition \ref{prop:reduce-to-bandits}, this is at most twice the bandit algorithm's regret on those rounds.
  So our expected regret totals at most $O\left(\ln\left(|\H|T/\delta\right) \sqrt{T}\right)$ to the best policy in $\Pi$.
  With probability $1-\delta$, $\Pi$ contains the optimal $\gamma$-fair classifier; with the remaining probability, the algorithm's regret to the best $\gamma$-fair classifier can be bounded by $T$.
  Choosing $\delta \leq \frac{1}{\sqrt{T}}$ gives the result.
\end{proof}

\subsection{Step 1: Constructing a Fair CSC Oracle From Exploration Data}

Let $S_E$ denote the set of $T_0$ labeled examples
$\{z_i = (x_i, a_i, y_i)\}_{i=1}^{T_0}$ collected from the initial
exploration phase, and let $\cD_E$ denote the empirical distribution
over $S_E$. We will use $\cD_E$ as a proxy for the true distribution
to form an \emph{empirical fairness constraint}. To support the learning
algorithm in the second phase, we need to construct an oracle that
solves CSC problems subject to the empirical fairness constraint.
Formally, an instance of the \emph{FairCSC} problem for the class $\H$
is given by a set of $n$ tuples
$\{(x_j, \cminus_j, \cplus_j)\}_{j=1}^n$ as before, along with a
fairness parameter $\gamma$ and an approximation parameter $\nu$.  We
wish to solve the following fair CSC problem:
\begin{align}
  \min_{\pi \in \Delta(\H)} \; \Ex{h\sim \pi}{ \sum_{j=1}^n \cgeneral{h(x_j)}_j} \quad \mbox{ such that } \quad |\Delta_{FPR}(\pi, \cD_E)|\leq \gamma
\end{align}
where
$\Delta_{FPR}(\pi, \cD_E) = FPR_{1}(\pi, \cD_E) - FPR_{-1}(\pi,
\cD_E)$ and each $FPR_{j}(\pi, \cD_E)$ denotes the false positive rate
of $\pi$ on distribution $\cD_E$. We show a useful structural property
that there always exists a small-support optimal solution; the proof appears in Appendix~\ref{sec:sparsity}.

\begin{lemma} \label{support_size1}
There exists an optimal solution for the FairCSC that is
  a distribution over $\H$ with support size no greater than 2.
\end{lemma}

We therefore consider the set of sparse convex combinations:
$$\Pi = \{\pi \in \Delta(\H) \mid \mathrm{Supp}(\pi)\leq 2, \quad|\Delta_{FPR}(\pi, \cD_E)|\leq \gamma + \beta\}$$
and focus on algorithms that only deploy policies from $\Pi$ and measure their performance with respect to $\Pi$. For any $\pi\in \Pi$, we will
write $\pi(h)$ to denote the probability $\pi$ places on $h$. Applying a 
standard concentration inequality, we can show (Lemma \ref{lemma:oracle-efp}) that each policy in
$\Pi$ is also approximately fair with respect to the underlying distribution.

We provide a reduction from FairCSC problems to standard CSC problems
as follows: 1) We first apply a standard transformation on the input
CSC objective to derive an equivalent weighted classification problem,
in which each example $j$ has importance weight
$|\cminus_j - \cplus_j|$.  2) We then run the fair classification
algorithm due to~\cite{MSR} that solves the weighted classification
problem \emph{approximately} using a polynomal number of CSC oracle
calls. 3) Finally, we follow an approach similar to that
of~\cite{cotter} to shrink the support size of the solution returned
by the fair classification algorithm down to at most 2, which can be
done in polynomial time.

\begin{theorem}[Reduction from FairCSC to CSC]\label{reduction}
  For any $0 < \nu < \gamma/2$, there exists an oracle-efficient
  algorithm that calls a CSC oracle for $\H$ at most $O(1/\nu^2)$ times
  and computes a solution $\hat \pi\in \Delta(\H)$ that has a support
  size of at most 2, satisfies $\gamma$-EFP, and has total cost
  \[
    \Ex{h\sim \hat\pi}{ \sum_{j=1}^n c_j^{h(x_j, a_j)}} \leq \min_{\pi
      \in \Pi} \Ex{h\sim \pi}{ \sum_{j=1}^n c_j^{h(x_j, a_j)}}
    +\epsilon
  \]
  with $\epsilon = 4\nu \sum_{j=1}^n |\cminus_j - \cplus_j|$.
\end{theorem}

\subsection{Step 2: The Adaptive Learning Phase}\label{feasible}
\paragraph{Overview of bandit algorithm.}
In the second phase, rounds $t > T_0$, we utilize a bandit algorithm
to make predictions.  We now describe the algorithm, which closely
follows the ILOVETOCONBANDITS algorithm by \cite{minimonster} but
with important modifications that are necessary to handle
approximation error in the FairCSC oracle.

At each round $t > T_0$, the bandit algorithm produces a distribution $Q_t$ over policies $\pi$.
Each policy $\pi$ is a convex combination of two classifiers in $\H$ and satisfies approximate fairness.
The algorithm then draws $\pi$ from $Q_t$, draws $h$ from $\pi$, and labels $\hat{y}_t = h(x_t)$.
To choose $Q_t$, the algorithm places some constraints on $Q$ and runs a short coordinate descent algorithm to find a $Q$ satisfying those constraints.
Finally, it mixes in a small amount of the uniform distribution over labels (which can be realized by mixing between $\plus$ and $\minus$).
We will see that the constraints, called the feasibility program, correspond to roughly bounding the expected regret of the algorithm along with bounding the variance in regret of each possible $\pi$.

\paragraph{Feasibility program.}
To describe the feasibility program, we first introduce some
notation. For each $t$, we will write $p_t$ to denote the probability
that prediction $\hat y_t$ is selected by the learner, and $\ell_t$ be
the incurred (contextual bandit) loss given by the transformation in
Proposition~\ref{prop:reduce-to-bandits}.

 for each policy $\pi\in \Pi$, let
\begin{align*}
  \hat L_t(\pi) = \frac{1}{t} \sum_{s=1}^t \ell_s \frac{\Pr[\pi(x_s) = \hat{y}_s]}{p_s}, \qquad
  L(\pi) = \Ex{(x, a, y)\sim\D}{\Ex{\pi}{\mathbf{1}[\pi(x) \neq y]}}
\end{align*}
denote the estimated average loss given by the \emph{inverse
  propensity score} (IPS) estimator and true expected loss for
$\pi$, respectively.  Similarly, let
\begin{align*}
  \widehat \Reg_t(\pi) = \hat L_t(\pi) - \min_{\pi' \in \Pi} \hat L_t(\pi'), \qquad
  \Reg(\pi) = L(\pi) - \min_{\pi'\in \Pi} L(\pi'),
\end{align*}
denote the estimated average regret and the true expected regret. In
order to bound the variance of the IPS estimators, we will ensure that
the learner predicts each label with minimum probability $\mu_t$ at
each round $t$. In particular, given a solution $Q$ for the program
and a minimum probability parameter $\mu_t$, the learner will predict
according to the mixture distribution $Q^{\mu_t}(\cdot \mid x)$ (a 
distribution that predicts $+1$ w.p. $\mu_t$, and predicts according
to $Q$ w.p. $1-\mu_t$):
\[
  Q^{\mu_t}(\hat y \mid x) =\mu_t + (1 - 2\mu_t) \int_{\pi \in \Pi}
  Q(\pi)\Pr[\pi(x) = \hat y] d\pi
\]
Note that this can be represented as a convex combination of
classifiers from $\H$ since we assume that $\plus \in \cH$. We define
for each $\pi\in \Pi$,
$b_t(\pi) = \frac{\widehat \Reg_t(\pi)}{4(e - 2)\mu_t\ln(T)}$, and
also initialize $b_0(\pi) = 0$.

We describe the feasibility program solved at each step. The approach
and analysis directly follow and extend that of \cite{minimonster}.
In that work, the first step at each round is to compute the best
policy so far, which lets us compute $\widehat{\Reg}_t(\pi)$ and
$b_t(\pi)$ for any policy $\pi$. Here, our FairCSC oracle only
computes approximate solutions, and so we can only compute regret
relative to the approximately best policy so far, which leads to
corresponding approximations $\widetilde{\Reg}_t(\pi)$ and
$\tilde{b}_t(\pi)$.  Then, our algorithm solves the same feasibility
program (although a few more technicalities must be handled): given
history $H_t$ (in the second phase) and minimum probability $\mu_t$,
find a probability distribution $Q$ over $\Pi$ such that
\begin{align}
\label{reg}\int_{\pi\in \Pi} Q(\pi) \tilde{b}_{t-1}(\pi)d\pi &\leq 4 \tag{Low regret}\\
\label{var}\forall \pi\in\Pi: \quad \Ex{x\sim H_t}{\frac{1}{Q^{\mu_t}(\pi(x) \mid x)}} &\leq 4 + \tilde{b}_{t-1}(\pi) \tag{Low variance}
\end{align}

Intuitively, the first constraint ensures that the estimated regret
(based on historical data) of the solution is at most
$\tilde O(1/\sqrt{t})$. The second constraint bounds the variance of
the resulting IPS loss estimator for policies in $\Pi$, which in turn
allows us to bound the deviation between the empirical regret and the
true regret for each policy over time. Importantly, we impose a tighter
variance constraint on policies that have lower empirical regret so far,
which prioritizes their regret estimation.

To solve the feasibility program using our FairCSC oracle, we will run
a coordinate descent algorithm (full description in
Section~\ref{feasibleproof}). Over iterations, the algorithm maintains
and updates a vector $Q$ of nonnegative weights that may sum to less
than one; at the end, the remaining probability mass is placed on the
empirically best policy $\hat \pi^t$ (computed using a single call of
FairCSC). At each iteration, the algorithm first checks whether the
current $Q$ violates the regret constraint; if so, the algorithm will
shrink all the weights to meet the regret constraint. If the regret
constraint is satisfied, the algorithm will then find the policy $\pi$
such that its variance constraint is most violated, which can be
identified using a single call of FairCSC oracle by the result
of~\cite{minimonster}. If the constraint violation is above 0, the
algorithm increases the weight $Q(\pi)$. The algorithm halts when all
of the constraints are satisfied. Lastly, the distribution output by this
computation is then mixed with a small amount of the uniform
distribution $\mu_t$ over labels.

In the
following, let $\Lambda_0 = 0$ and for any $t \geq 1$,
\[ \Lambda_t := \frac{\nu}{4(e-2)\mu_t^2 \ln(T)} . \] where $\nu$ is
the approximation parameter of the FairCSC oracle.
\begin{lemma}\label{accuracy}
  Algorithm~\ref{alg:coor} halts in a number of iterations (and oracle
  calls) that is polynomial in $\frac{1}{\mu_t}$, and outputs a weight
  vector $Q$ that is a probability distribution with the following
  guarantee:
  \begin{align*}
    \int_{\pi \in \Pi} Q(\pi) (4 + b_{t-1}(\pi))d\pi &\leq 4 + \Lambda_t  \\
    \forall \pi\in\Pi: \quad \Ex{x\sim H_t}{\frac{1}{Q^{\mu_t}(\pi(x) \mid x)}} &\leq 4 + b_{t-1}(\pi) + \Lambda_t .
  \end{align*}
\end{lemma}

\subsection{Step 3: Regret Analysis}\label{regret}
The key step in our regret analysis is to establish a tight relationship
between the estimated regret and the true expected regret and show
that for any $\pi\in \Pi$,
$\Reg(\pi) \leq 2\widehat \Reg(\pi) + \epsilon_t$, with
$\epsilon_t = \tilde O(1/\sqrt{t})$. The final regret guarantee then
essentially follows from the guarantee of Lemma~\ref{accuracy} that
the estimated regret of our policy is bounded by
$\tilde O\left(1/t\right)$ with proper setting of $\mu_t$.

To bound the deviation between $\Reg(\pi)$ and $\widehat\Reg_t(\pi)$,
we need to bound the variance of our IPS estimators. Let us define the
following for any probability distribution $P$ over $\Pi$,
$\pi\in \Pi$,
\begin{align*}
  V(P, \pi, \mu) := \Ex{x \sim \D}{\frac{1}{P^\mu(\pi(x) \mid x)}}\qquad \mbox{ }\qquad
  \hat V_t(P, \pi, \mu) := \Ex{x \sim H_t}{\frac{1}{P^\mu(\pi(x) \mid x)}}
\end{align*}

Recall that through the feasibility program, we can directly bound
$\hat V_t(Q_t, \pi, \mu_t)$ for each round. However, to apply a 
concentration inequality on the IPS estimator, we need to bound the
population variance $V(Q_t, \pi, \mu_t)$. We do that through a
deviation bound between $\hat V_t(Q_t, \pi, \mu_t)$ and
$V(Q_t, \pi, \mu_t)$ for all $\pi \in \Pi$. In particular, we rely on
the sparsity on $\Pi$ and apply a covering argument. Let
$\Pi_\eta\subset \Pi$ denote an $\eta$-cover such that for every $\pi$
in $\Pi$,
$\min_{\pi' \in \Pi_\eta} \| \pi(h) - \pi'(h)\|_\infty \leq \eta$ for
any $h\in \H$. Since $\Pi$ consists of distributions with support size
at most 2, we can take the cardinality of $\Pi_\eta$ to be bounded by
$\lceil|\H|^2 /\eta\rceil$.

\begin{claim} \label{cla:cover} Let $P$ be any distribution over the
  policy set $\Pi$, and let $\pi$ be any policy in $\Pi$. Then there
  exists $\pi'\in \Pi_\eta$ such that
  $|V(P, \pi, \mu) - V(P, \pi' ,\mu)|_{\infty}, |\hat V_t(P, \pi, \mu) - \hat
  V_t(P, \pi' ,\mu)|_{\infty} \leq \frac{\eta}{\mu(\mu + \eta)}.$
\end{claim}

\begin{lemma}\label{var_dev}
  Suppose that
  $\mu_t \geq \sqrt{\frac{\ln(2|\Pi_{\eta}|t^2/\delta)}{2t}}, t \geq
  8\ln(2|\Pi_{\eta}|t^2/\delta)$.  Then with probability $1 - \delta$,
 $$V(P, \pi, \mu_t) \leq 6.4 \hat V_t(P, \pi, \mu_t) + 162.6 + \frac{2\eta}{\mu_t(\mu_t+\eta)}$$
\end{lemma}

\noindent Next we bound the deviation between the estimated loss and
true expected loss for every $\pi\in\Pi$.

\begin{lemma}\label{loss_bound}
  Assume that the algorithm solves the per-round feasibility program
  with accuracy guarantee of Lemma~\ref{accuracy}. With probability at
  least $1 - \delta$, we have for all $t\in [T]$ all policies
  $\pi\in \Pi$, $\lambda\in [0, \mu_t]$, and
  $t \geq 8 \ln(2|\Pi_\eta|t^2/\delta)$,
  \[
    |L(\pi) - \hat L_t(\pi)| \leq (e - 2)\lambda \left(188.2 +
      \frac{1}{t} \sum_{s=1}^t \left(6.4 b_{s-1}(\pi) + 6.4
        \Lambda_{s-1} +\frac{2\eta}{\mu_{s}(\mu_{s} + \eta)}
      \right)\right) + \frac{\ln\left(\frac{|\Pi_\eta|
          T}{\delta}\right)}{\lambda t}
  \]
\end{lemma}

To bound the difference between $\Reg(\pi)$ and
$\widehat \Reg_t(\pi)$, we will set $\eta=1/T^2$,
$\mu_t = \frac{3.2\ln(|\Pi_\eta|T/\delta)}{\sqrt{t}}$ the
approximation parameter $\nu$ of FairCSC to be $1/T$.

\begin{lemma}\label{bigproof}
  Assume that the algorithm solves the per-round feasibility program
  with the accuracy guarantee of Lemma~\ref{accuracy}. With probability at
  least $1 - \delta$, we have for all $t\in [T]$ all policies
  $\pi\in \Pi$, and for all $t \geq 8 \ln(2|\H|^2 T^3/\delta)$,
  \[
    \Reg(\pi) \leq 2\widehat \Reg_t(\pi) + \epsilon_t, \qquad \mbox{and}\qquad
    \widehat\Reg_t(\pi) \leq 2 \Reg(\pi) + \epsilon_t
  \]
  with $\epsilon_t = \frac{1000 \ln(|\H|^2 T^2/\delta)}{\sqrt{t}}$.
\end{lemma}

\begin{theorem} \label{thm:bandit-regret}
  The bandit algorithm, given access to an approximate-CSC oracle, runs in time polynomial in $T$ and achieves expected regret at most $O\left(\ln(|\H|T/\delta) ~ \sqrt{T}\right)$.
\end{theorem}

\section{Lower Bound} \label{lower_res}
In this section we show that the tradeoff that our algorithm exhibits between its regret bound and the ``fairness gap'' $\gamma' - \gamma$ (i.e. our algorithm is $\gamma'$-fair, but competes with the best $\gamma$-fair classifier when measuring regret) is optimal. We do this by constructing a lower bound instance consisting of two very similar distributions, $\D_1$ and $\D_2$ defined as a function of our algorithm's fairness target $\gamma$. The instance is defined over a simple hypothesis class $\H$.  $\H$ contains the two constant classifiers ($\minus$ and $\plus$), and a pair of classifiers ($h_1$ and $h_2$) that each guarantee low error on both distributions, but only one of which satisfies the 0-EFP constraint. Informally, we first prove that the two distributions cannot be distinguished for at least $\Theta(\frac{1}{\gamma^2})$ rounds. We then argue that any algorithm satisfying our $\gamma$-EFP$(\delta)$ constraint must deploy $\minus$ or $\plus$ with substantial probability over these initial rounds in order to guarantee that it does not violate its fairness guarantee on either $\D_1$ or $\D_2$. However, this implies incurring linear regret per round during this phase, which leads to our lower bound.

\begin{theorem} \label{thm:main-negative1}
  Fix any $\alpha \in (0,0.5)$ and let $T \geq \sqrt[\alpha]{16}$. Fix any $\delta \leq 0.24$. There exists a hypothesis class $\H$ containing $\{\pm \mathbf{1}\}$ such that any algorithm satisfying a $T^{-\alpha}$-EFP($\delta$) fairness constraint has expected  regret with respect to the set of 0-EFP fair policies of $\Omega\left(T^{2\alpha}\right)$.
\end{theorem}

In order to prove theorem \ref{thm:main-negative1}, we make use of a couple of standard tools:

\begin{lemma} \label{lem:KL} (Pinsker's Inequality)
Let $\mathcal{D}_1$, $\mathcal{D}_2$ be probability distributions. Let $A$ be any event. Then:
\begin{equation*}
\left|\mathcal{D}_1(A)-\mathcal{D}_2(A)\right| \leq \sqrt{\frac{1}{2}KL(\mathcal{D}_1||\mathcal{D}_2)}
\end{equation*}
\end{lemma}

The following is a simple corollary that follows from the additivity of KL-divergence over product distributions.
\begin{corollary} \label{cor:KL}
Let $t \in \mathbb{N}$. Consider the product distributions $\mathcal{D}_1^t$, $\mathcal{D}_2^t$. For any event $A$,
\begin{equation*}
\left|\mathcal{D}_1^t(A)-\mathcal{D}_2^t(A)\right| \leq \sqrt{\frac{1}{2}t\cdot KL(\mathcal{D}_1||\mathcal{D}_2)}
\end{equation*}
\end{corollary}

Next, for any algorithm $\mathcal{A}$, round $t$, hypothesis $h$, and distribution $\D$, let
  \[ P_t(h,\D) = \Prob\left[\text{$\mathcal{A}$ plays $h$ on round $t$}\right] \]
when given inputs from $\D$.
We say an algorithm $(\beta,t,h)$-distinguishes distributions $\D_1$ and $\D_2$ if
  \[ \left| P_t(h,\D_1) - P_t(h,\D_2) \right| > \beta . \]

\begin{lemma} \label{lem:lower-no-disting}
Let $\D_1, \D_2$ be two probability distributions. No algorithm can $(\beta,t,h)$-distinguish $\D_1$ and $\D_2$ for any $h$ and $t \leq \frac{2\beta^2}{KL(\mathcal{D}_1||\mathcal{D}_2)}$.
\end{lemma}

\begin{proof}
Assume towards a contradiction that there exists an algorithm that $(\beta,t,h)$-distinguishes $\D_1$ and $\D_2$ for some $h$ and $t \leq \frac{2\beta^2}{KL(\mathcal{D}_1||\mathcal{D}_2)}$. This defines an event $A$ such that
\begin{equation*}
|\D_1^t(A) - \D_2^t(A)| > \beta \geq \sqrt{\frac{1}{2}tKL(\D_1||\D_2)}
\end{equation*}
which contradicts corollary \ref{cor:KL}.
\end{proof}

With these tools in hand, we are ready to prove the lower bound (Theorem \ref{thm:main-negative1}):

\begin{proof} [Proof of Theorem \ref{thm:main-negative1}]
Fix any $\alpha \in (0,0.5)$ and let $T \geq \sqrt[\alpha]{16}$. Denote $\gamma = T^{-\alpha}$. Fix any $\delta \leq 0.24$.

Define the following distributions over (X,A,Y):

$\mathcal{D}_1$ given by:
\begin{table}[H]
\centering
\begin{tabular}{ll|lllll}
                                                 &                & $x_1$          & $x_2$          & $x_3$   & $x_4$ \\ \hline

                        & $\Prob[(x,a)]$     & $1/8$           & $1/8$           & $1/8$   & $1/8$ \\

\multirow{-2}{*}{$A=-1$} & $\Prob[y=1|(x,a)]$ & $0.5+4\gamma$ & $0.5-4\gamma$ & $1$     & $0$ \\ \hline

                         & $\Prob[(x,a)]$     & $1/8$           & $1/8$           & $1/8$   & $1/8$ \\

\multirow{-2}{*}{$A=+1$} & $\Prob[y=1|(x,a)]$ & $0.5-4\gamma$ & $0.5+4\gamma$ & $1$     & $0$
\end{tabular}
\end{table}

$\mathcal{D}_2$ given by:
\begin{table}[H]
\centering
\begin{tabular}{ll|lllll}
                         &                    & $x_1$           & $x_2$           & $x_3$   & $x_4$ \\ \hline

                         & $\Prob[(x,a)]$     & $1/8$           & $1/8$           & $1/8$   & $1/8$ \\

\multirow{-2}{*}{$A=-1$} & $\Prob[y=1|(x,a)]$ & $0.5+4\gamma$ & $0.5-4\gamma$ & $1$     & $0$ \\ \hline

                         & $\Prob[(x,a)]$     & $1/8$           & $1/8$           & $1/8$   & $1/8$ \\

\multirow{-2}{*}{$A=+1$} & $\Prob[y=1|(x,a)]$ & $0.5+4\gamma$ & $0.5-4\gamma$ & $1$     & $0$
\end{tabular}
\end{table}

The available hypotheses $\mathcal{H} = \{\minus,\plus,h_1,h_2\}$ are defined as:
\begin{table}[H]
\centering
\begin{tabular}{cc|ccccc}
                       		&        	& $x_1$ & $x_2$ & $x_3$ & $x_4$ \\ \hline
\multirow{2}{*}{$\minus$}  	& $A=-1$ 	& $-1$  & $-1$  & $-1$  & $-1$  \\
                       		& $A=+1$  	& $-1$  & $-1$  & $-1$  & $-1$ \\ \hline
\multirow{2}{*}{$\plus$}  	& $A=-1$ 	& $+1$  & $+1$  & $+1$  & $+1$ \\
                       		& $A=+1$  	& $+1$  & $+1$  & $+1$  & $+1$ \\ \hline
\multirow{2}{*}{$h_1$} 		& $A=-1$ 	& $+1$  & $-1$  & $+1$  & $-1$  \\
                       		& $A=+1$  	& $+1$  & $-1$  & $+1$  & $-1$ \\ \hline
\multirow{2}{*}{$h_2$} 		& $A=-1$		& $+1$  & $-1$  & $+1$  & $-1$   \\
                       		& $A=+1$  	& $-1$  & $+1$  & $+1$  & $-1$
\end{tabular}
\end{table}

The performance of the hypotheses in $\mathcal{H}$ on the two distributions is given by:
On $\mathcal{D}_1$:
\begin{table}[H]
\centering
\begin{tabular}{l|l|l}
      		& $L_{\mathcal{D}}^{0-1}(h)$     & $\Delta_{FPR}(h) $     \\ \hline
$\minus$  	& $0.5$         		& $0$         \\ \hline
$\plus$   	& $0.5$           	& $0$         \\ \hline
$h_1$ 		& $0.25$          	& $4\gamma$ \\ \hline
$h_2$ 		& $0.25-2\gamma$ 	& $0$         \\
\end{tabular}
\end{table}

On $\mathcal{D}_2$:
\begin{table}[H]
\centering
\begin{tabular}{l|l|l}
      		& $L_{\mathcal{D}}^{0-1}(h)$      & $\Delta_{FPR}(h) $    \\ \hline
$\minus$  	& $0.5$           	& $0$         \\ \hline
$\plus$   	& $0.5$           	& $0$         \\ \hline
$h_1$ 		& $0.25-2\gamma$ 	& $0$         \\ \hline
$h_2$ 		& $0.25$          	& $4\gamma$ \\
\end{tabular}
\end{table}

Note that on both distributions, $h_1$ and $h_2$ both have substantially lower error than the two constant classifiers, but only one of which satisfies the $\gamma$-fairness constraint --- and which one of them it is depends on whether the underlying distribution is $\D_1$ or $\D_2$. Note also that one of them always satisfies a 0-fairness constraint, and so sets the benchmark for 0-EFP regret. The main fact driving our lower bound is that until the algorithm can reliably distinguish $\D_1$ from $\D_2$, it must  place substantial weight on the constant classifiers, incurring high regret.

We first establish that the two distributions are hard to distinguish by showing that the KL-divergence between $\mathcal{D}_1$, $\mathcal{D}_2$ is bounded by $O(\gamma^2)$:
\begin{align*}
KL(\mathcal{D}_1||\mathcal{D}_2) &= \frac{2}{8}\left(\frac{1+8\gamma}{2}\ln\left(\frac{1+8\gamma}{1-8\gamma}\right)+\frac{1-8\gamma}{2}\ln\left(\frac{1-8\gamma}{1+8\gamma}\right)\right)\\
&= \gamma \ln\left(\frac{\frac{1+8\gamma}{1-8\gamma}}{\frac{1-8\gamma}{1+8\gamma}}\right)\\
&= \gamma \ln\left(\left(\frac{1+8\gamma}{1-8\gamma}\right)^2\right)\\
&= 2\gamma \ln \left(\frac{1+8\gamma}{1-8\gamma}\right)\\
&= 2\gamma \ln \left(1+\frac{16\gamma}{1-8\gamma}\right)\\
&\leq 2\gamma\frac{16\gamma}{1-8\gamma}\\
&= \frac{64\gamma^2}{2(1-8\gamma)}\\
&\leq 64\gamma^2
\end{align*}

Let $\mathcal{A}$ be a $\gamma$-EFP($\delta$) algorithm. Let $K = \frac{0.01^2}{32\gamma^2}$ (and note that, for $\alpha \in (0,0.5)$, $K = \frac{0.01^2}{32\gamma^2} = \frac{0.01^2 T^{2\alpha}}{32} < \frac{0.01^2 T}{32} \leq T$).
  Let $t \leq K$ (note that the number of samples observed by time $t$ is $t' \leq t$); then by lemma \ref{lem:lower-no-disting},
  $$P_t(h_1,\D_2) \leq P_t(h_1,\D_1) + 0.01$$
  $$P_t(h_2,\D_1) \leq P_t(h_2,\D_2) + 0.01$$

Observe that any convex combination $\pi$ of classifiers deployed under $\D_1$ fails to satisfy the $\gamma$-EFP constraint unless it puts weight less than $1/4$ on $h_1$. Similarly,  any convex combination $\pi$ of classifiers deployed under $\D_2$ fails to satisfy the $\gamma$-EFP constraint unless it puts weight less than $1/4$ on $h_2$. Since by definition, a $\gamma$-EFP$(\delta)$ algorithm deploys only $\gamma$-EFP policies on any distribution it is deployed on except with probability $\delta$, we have that for all $t \in [T]$
  $$P_t(h_1,\D_1) \leq \frac{1}{4} + \delta$$
  $$P_t(h_2,\D_2) \leq \frac{1}{4} + \delta$$

And thus
  $$P_t(h_1,\D_2) \leq P_t(h_1,\D_1) + 0.01 = 0.25 + 0.01 + \delta = 0.26 + \delta$$
  $$P_t(h_2,\D_1) \leq P_t(h_2,\D_2) + 0.01 = 0.25 + 0.01 + \delta = 0.26 + \delta$$

Hence on either distribution, we have,
$$\Prob[\text{$\mathcal{A}$ deploys $\plus$ or $\minus$}~on~round~t] \geq 1 - (0.25 + \delta) - (0.26 + \delta) = 0.49 - 2\delta$$

  The best performing 0-EFP policy on $\D_1$ is $h_2$, while on $\D_2$ it is $h_1$. Both of these induce expected per-round loss of less than $\frac{1}{4}$. Since the expected per round loss of either $\plus$ or $\minus$ is $\frac{1}{2}$ on both distributions, if $\plus$ or $\minus$ are deployed with constant probability, the expected per-round regret incurred is a constant bounded away from zero. As a result, the expected 0-EFP regret of $\mathcal{A}$ is at least $\Omega(K) = \Omega\left(\frac{1}{\gamma^2}\right)$.

The result is that any $T^{-\alpha}$-EFP($\delta$) algorithm must have expected 0-EFP regret of $\Omega(T^{2\alpha})$.
\end{proof}

\section*{Acknowledgments}
We thank Nati Srebro for a conversation leading to the question we
study here. We thank Michael Kearns for helpful discussions at an
early stage of this work.  YB and KL were funded in part by Israel
Science Foundation (ISF) grant 1044/16, the United States Air Force
and DARPA under contract FA8750-16-C-0022, and the Federmann Cyber
Security Center in conjunction with the Israel national cyber
directorate.  AR was funded in part by NSF grant CCF-1763307 and the
United States Air Force and DARPA under contract FA8750-16-C-0022.
ZSW was funded in part by a Google Faculty Research Award, a
J.P. Morgan Faculty Award, a Mozilla research grant, and a Facebook
Research Award. Part of this work was done while KL and ZSW were
visiting the Simons Institute for the Theory of Computing, and BW was
a postdoc at the University of Pennsylvania's Warren Center and at
Microsoft Research, New York City. Any opinions, findings and
conclusions or recommendations expressed in this material are those of
the authors and do not necessarily reflect the views of JP Morgan, the
United States Air Force and DARPA.

\bibliographystyle{plain}
\bibliography{equalOpp}

\newpage
\appendix

\section{Proof of Proposition~\ref{prop:reduce-to-bandits}}

We briefly recall the contextual bandits setting below, for an arbitrary loss function:
\\~\\
\noindent\fbox{%
    \parbox{\textwidth}{%
\centerline{Online Learning in the Contextual Bandits Setting}
\begin{algorithmic}
\FOR {$t=1,...,T$}
	\STATE Learner chooses a convex combination $\pi_t \in \convH$.
	\STATE Environment draws $(x_t,y_t) \sim \mathcal{D}$ independently, learner observes $x_t$.
	\STATE Learner labels the point $\hat{y}_t = h_t(x_t)$, where $h_t \sim \pi_t$.
    \STATE Learner observes loss $\ell(\hat{y}_t, y_t) \in [0,1]$.
\ENDFOR
\end{algorithmic}
    }%
}
~\\

\begin{proof}[Proof of Proposition~\ref{prop:reduce-to-bandits}]
  Consider the following transformed loss matrix:
  $$
  \tilde{L} =
  \begin{blockarray}{ccc}
   & y = +1 & y = -1 \\
  \begin{block}{c(cc)}
    \hat y = +1 & 0 & 2 \\
    \hat y = -1  & 1 & 1  \\
  \end{block}
  \end{blockarray}
  $$
  Given an online learning with partial feedback problem, we instantiate the bandit algorithm and always play the action it recommends.
  We then provide the algorithm with its feedback $\tilde{L}_{\hat{y}^t,y^t}$.
  This is possible because if $\hat{y}^t = +1$ then we observe $y^t$, and if not then the feedback is $1$ regardless of the unobserved value of $y^t$.
  For a sequence of arrivals $S = \{(x^t,y^t)\}_{t=1}^T$, let $m(S)$ be the number of arrivals with $y^t = -1$.
  Let $L(\pi,S) = \sum_{t=1}^t L_{\hat{y}^t,y^t}$ and similarly for $\tilde{L}(\pi,S)$.
  Then we have for all $\pi,S$ that $\tilde{L}(\pi,S) = L(\pi,S) + m(S)$.
  In other words, on each round where $y^t=1$, a prediction experiences the same loss under $L$ and under $\tilde{L}$; and on each round where $y^t=-1$, the loss is exactly one larger in the bandit setting.
  This difference does not depend on the prediction of the hypothesis, therefore every policy's total loss under the bandit loss is exactly $m(S)$ larger than under the original loss.

  It follows that our algorithm's regret is exactly equal to the bandit algorithm's.
  Finally, a bookkeeping note: in order that losses be bounded in $[0,1]$, we must repeat the above argument using $0.5\tilde{L}$ in place of $\tilde{L}$, which simply scales the bandit algorithm's regret by $0.5$ relative to our algorithm's.
\end{proof}

\section{Missing Proofs for Section~\ref{sec:comp}}

\subsection{Proof of  Lemma~\ref{support_size1}}\label{sec:sparsity}
In this subsection, we establish a useful structural property for the
general problem minimizing linear loss function subject to fairness
constraints. This in turn provides a proof for
Lemma~\ref{support_size}. In particular, given a hypothesis class
$\H$, and a training set of labelled samples $S$, vectors
$a, b\in \mathbb{R}^{|\H|}$, consider the problem:
\begin{equation*}
\begin{aligned}
  & \min_{x \in \Delta(|\H|)}
    & & a^\intercal x \\
    & \text{subject to}
    & & b^\intercal x \leq \gamma \\
    & & & b^\intercal x \geq -\gamma
\end{aligned}
\end{equation*}

Note that both the problem of weighted classification or
cost-sensitive classification can be viewed as an instantiation of the
linear program defined above. The sparsity in the solution will be
useful in our analysis.
\begin{theorem}\label{support_size}
  In the linear program above, there exists an optimal solution that
  is a distribution over $\H$ with support size no greater than 2.
\end{theorem}

\begin{proof}
  Consider the following embedding of $\H$ in $\mathbb{R}^2$:
  $\forall h\in\H: \phi(h) = (a_h, b_h)$. Let
  $A = \{{\phi(h)}\mid \pi \in \H\}$.  Then the optimization problem
  can be written as the following problem over the convex hull
  $\mathrm{conv}(A)$:
\begin{equation*}
\begin{aligned}
& \underset{(z_1,z_2) \in \mathrm{conv}(A)}{\text{minimize}}
& & z_1 \\
& \text{subject to}
& & z_2 \leq \gamma \\
& & & z_2 \geq -\gamma
\end{aligned}
\end{equation*}

Note there exists an optimal solution $z^*$ that lies on an edge of
the polytope defined by $\mathrm{conv}(A)$. This means $z^*$ is either
a vertex of $\mathrm{conv}(A)$ or can be written as a convex
combination of two vertices of $\mathrm{conv}(A)$, say $z'$ and
$z''$. In the former case, $z^*$ can be induced by a single hypothesis
$h^*\in \H$, and in the latter case we know there exist
$h', h''\in \H$ such that $z' = \phi(h')$ and $z'' = \phi(h'')$. This
means the optimal solution $z^*$ can be induced by a convex
combination of hypotheses.
\end{proof}

Then the result of Lemma~\ref{support_size} follows immediately.

\subsection{Proof of Theorem~\ref{reduction}}

As mentioned, a standard concentration inequality immediately implies:
\begin{lemma} \label{lemma:oracle-efp}
  With probability $1 - \delta$, as long as $T_0 \geq c \sqrt{T\ln(|\H|/\delta)}$ for some universal constant $c > 0$, we have the following.
 First, every policy in $\Pi$ satisfies
  $\gamma + 2\beta$-EFP, and second, every support-$2$ $\gamma$-EFP policy is in $\Pi$, for $\beta = O\left(\sqrt{\ln(|\H|/\delta)}/T^{1/4}\right)$.
\end{lemma}

Recall that we collect a set of $T_0$ labeled examples
$\{z_i = (x_i, a_i, y_i)\}_{i=1}^{T_0}$ during the initial exploration
phase, and let $\cD_E$ denote the corresponding empirical
distribution. Recall that $\H$ is a hypothesis class defined over both
the features and the protected group memberships. We assume that $\H$
contains a constant classifier (which implies that there is at least
one fair classifier to be found, for any distribution). To simplify
notation, we consider hypotheses that labels each example with either
0 or 1.

Suppose that we are given a cost-sensitive classification
instance $(X_j, C_j^{1}, C_j^{0} )$. We would like to compute a
distribution over classifiers from $\cH$ that minimizes total cost
subject to the false positive rate fairness constraint. In particular,
consider the following {\em fair cost-sensitive classification (CSC)}
problem:
\begin{align}
  &  \min_{\pi \in \Delta(\cH)}\; \Ex{h\sim \pi}{ \sum_{j=1}^n (C_j^{1} h(X_j) + C_j^0 (1 - h(X_j)))}\\
  \mbox{such that } \forall j\in \{\pm 1\} \qquad &
                                                    \fpdisp_j(\pi) - \fpdisp_{-j}(\pi) \leq \gamma.\label{love}
\end{align}
$\fpdisp_j(\pi) = \Ex{h\sim \pi}{\fpdisp_j(h)}$. We will write
$\OPT_C$ to denote the objective value at the optimum for the problem,
that is the minimum cost achieved by a $\gamma$-EFP policy over
distribution $\cD_E$.

Equivalently, we can consider optimizing the following objective
function:
\[
  \min_{\pi \in \Delta(\cH)} \; \Ex{h\sim \pi}{ \sum_{j=1}^n
    W_i\,\mathbf{1}\{h(X_j) \neq Y_j \}}
\]
where each $W_j = |C_j^0 - C_j^1|$, $Y_j = 1$ if $C_j^0 > C_j^1$ and
$Y_j=0$ otherwise. To reduce the problem further to the same
formulation of~\cite{MSR}, we consider objective with normalized
weights
\[
  \min_{\pi \in \Delta(\cH)} \; \Ex{h\sim \pi}{ \sum_{j=1}^n
    w_j\,\mathbf{1}\{h(X_j) \neq Y_j \}}
\]
such that each $w_j = W_j / (\sum_j W_j)$. To simplify notation, we
will write
$err(h, \cP) = \sum_{j=1}^n w_j\,\mathbf{1}\{h(X_j) \neq Y_j \}$, and
$\OPT$ to denote optimal objective subject to $\gamma$-EFP.

For each of the fairness constraint in ~\eqref{love}, we will
introduce a dual variable $\lambda_j\geq 0$. This allows us to define
the partial Lagrangian of the problem:
\[
  \cL(\pi, \lambda) = \Ex{h\sim \pi}{err(h, \cP)} + \sum_{j\in \{\pm 1\}}
  \lambda_j \left(\fpdisp_j(\pi) - \fpdisp_{-j}(\pi) - \gamma\right)
\]
By strong duality, we have
\[
  \OPT = \min_{g\in \Delta(\cH)} \max_{\lambda\in \mathbb{R}_+^2}
  \cL(g, \lambda) = \max_{g\in \Delta(\cH)} \min_{\lambda\in
    \mathbb{R}_+^2} \cL(g, \lambda).
\]
where $\OPT$ is the optimal objective value of the ERM problem.

\cite{MSR} provide an oracle-efficient algorithm for finding a
$\nu$-\emph{approximate saddle point} $(\hat g, \hat \lambda)$ of the Lagrangian:
\begin{align*}
\cL(\hat \pi, \hat \lambda) &\leq \cL(g, \hat \lambda) + \nu \qquad \mbox{for all } g\in\Delta(\cH)\\
\cL(\hat \pi, \hat \lambda) &\geq \cL(\hat \pi, \lambda) - \nu \qquad \mbox{for all } \lambda\in \Lambda
\end{align*}
In their result, the algorithm restricts the dual space to be
$\Lambda = \{\|\lambda\|_1\leq B \mid\lambda \in \mathbb{R}^2_+\}$ for
some sufficiently large constant $B$. Their convergence rate and
approximation parameter both depend on such $C$. We show that under
the assumption that $\cH$ constains the two classifiers
$\mathbf{1}[a = j]$ for all $j\in\{\pm 1\}$, it is sufficient to set
$C = 2$, and thus restrict the dual space to be
\[
  \Lambda = \{\|\lambda\|_1\leq 2\mid\lambda \in \mathbb{R}^2_+\}
\]
Consequently, we can use their algorithm to find a $\nu$-approximate
saddle point with only $O\left(\frac{1}{\nu^2}\right)$ number of calls
to the oracle CSC$(\cH)$.

\begin{lemma}[Follows from Theorem 1 of~\cite{MSR}]
  There is an oracle-efficient algorithm that computes a
  $\nu$-approximate saddle point for the restricted Lagrangian with
  $\Lambda \{\|\lambda\|_1\leq 2\mid\lambda \in \mathbb{R}^2_+\}$,
  using $O\left(1/\nu^2 \right)$ calls to a CSC oracle over $\cH$.
\end{lemma}

Moreover, the approximate saddle point provides an approximate
solution for our problem.

\begin{lemma}
  Suppose that the class $\cH$ contains the two classifiers
  $\mathbf{1}[a = j]$ for all $j$ and that $(\hat \pi, \hat \lambda)$
  is a $\nu$-approximate saddle point of the Lagrangian. Then the
  distribution $\hat \pi$ satisfies
  \[
    err(\hat \pi, \cP) \leq \OPT +2 \nu, \qquad \mbox{and } \qquad
    \forall j\in \{\pm 1\} \qquad \fpdisp_j(\hat g) -
    \fpdisp_{-j}(\hat g) \leq \gamma + 2\nu.
  \]
\end{lemma}

\begin{proof}
  Let $\pi^*$ be the optimal feasible solution for the fair ERM problem.
  First, by the definition of approximate saddle point, we know that
\begin{align*}
  err(\hat \pi, \cP)& = \cL(\hat \pi, \mathbf{0})\\
& \leq \max_{\lambda \in
    \Lambda} \cL(\hat \pi, \lambda) \\
& \leq \cL(\hat \pi, \hat \lambda) +
  \nu\\
& \leq \min_{\pi\in \Delta(\cH)} \cL(\pi, \hat \lambda) + 2\nu \\
& \leq \cL(\pi^* , \hat \lambda) + 2\nu = \OPT + 2\nu
\end{align*}
where the equality follows from the fact that
$\cL(\pi^*, \hat \lambda) = \OPT$.

Next, we will bound the fairness constraint violations.  Suppose
without loss of generality that the following fairness constraint is
violated: $\fpdisp_1(\hat \pi) - \fpdisp_{-1}(\hat \pi)= \gamma + \alpha$
for some $\alpha \geq 0$. Let $\lambda'\in \Lambda$ such that
$\lambda'_1 = 2$. Then
\[
  \cL(\hat \pi, \hat \lambda) + \nu \geq \cL(\hat \pi, \lambda') =
  err(\hat \pi, \cP) + 2\alpha
\]
\noindent Thus, by the assumption of approximate saddle point,
\[
  err(\hat \pi, \cP) \leq \cL(\hat \pi, \hat \lambda) + \nu - 2\alpha \leq
  \cL(\pi^*, \hat \lambda) + 2\nu - \alpha = \OPT + 2\nu - 2\alpha.
\]

Now consider a distribution $\pi'$ that is defined as the mixture of
\[
\pi' = (1 - \alpha) \hat \pi + \alpha \mathbf{1}[a = -1].
\]
This means
\begin{align*}
  \fpdisp_1(\pi') &= (1 - \alpha) \fpdisp_1(\hat \pi) + \alpha \fpdisp_1(\mathbf{1}[a = -1]) = (1 - \alpha) \fpdisp_1(\hat \pi) \\
  \fpdisp_{-1}(\pi') &= (1 - \alpha) \fpdisp_{-1}(\hat \pi) + \alpha \fpdisp_{-1}(\mathbf{1}[a = -1]) = (1 - \alpha) \fpdisp_{-1}(\hat \pi) + \alpha
\end{align*}
It follows that
\[
  \fpdisp_1(\pi') - \fpdisp_{-1}(\pi') = (1 - \alpha) (\gamma + \alpha) - \alpha \leq \gamma
\]
which implies that $\pi'$ is a feasible solution for the fair ERM
problem. This implies that
\[
err(\hat \pi, \cP) \geq err(\pi', \cP) - \alpha \geq \OPT - \alpha
\]

Thus, we have $\OPT + 2\nu - 2\alpha \geq \OPT - \alpha$, which
implies that $\alpha \leq 2\nu$. This completes the proof.
\end{proof}

To facilitate our analysis, we would like a solution $\hat \pi$ that
satisfies the fairness cosntraint without any violation. To achieve
this, we simply tighten the constraint by an amount of $2\nu$ and
compute the $\nu$-approximate saddle point for the tightened
Lagrangian, replacing $\gamma$ with $\gamma'= \gamma - 2\nu$. We also
need ensure such tightening of the constraint does not severely
increase the resulting error.

\begin{lemma}[Bound on additional error from tightening.]
  Suppose that $\gamma > 2\nu$.  Let $\OPT'$ be the objective value at
  the optimum for the tighted optimization problem:
\begin{align*}
  &\min_{\pi \in \Delta(\cH)} \; \Ex{h\sim \pi}{ \sum_{j=1}^n
    w_j\,\mathbf{1}\{h(X_j) \neq Y_j \}}\\
  \mbox{such that } \forall j\in \{\pm 1\} \qquad &
                                                    \fpdisp_j(\pi) - \fpdisp_{-j}(\pi) \leq \gamma - 2\nu
\end{align*}
Then as long as that the class $\cH$ contains the two classifiers
$\mathbf{1}[a = j]$ for both $j\in\{\pm 1\}$,
$\OPT' - \OPT \leq 2\nu$.
\end{lemma}

\begin{proof}
  Let $\pi^*$ be an optimal solution to the original (un-tightened)
  problem.  Suppose without loss of generality that the following
  fairness constraint is violated:
  $\fpdisp_1(\hat \pi) - \fpdisp_{-1}(\hat \pi)\geq 0$. Now consider a
  distribution $\pi'$ that is defined as the mixture of
\[
  \pi' = (1 - 2\nu) \pi^* + 2\nu \mathbf{1}[a = -1].
\]
Consequently, we can write
\begin{align*}
  \fpdisp_1(\pi') &= (1 - 2\nu) \fpdisp_1(\pi^*) + 2\nu \fpdisp_1(\mathbf{1}[a = -1]) = (1 - 2\nu) \fpdisp_1(\pi^*) \\
  \fpdisp_{-1}(\pi') &= (1 - 2\nu) \fpdisp_{-1}(\pi^*) + 2\nu \fpdisp_{-1}(\mathbf{1}[a = -1]) = (1 - 2\nu) \fpdisp_{-1}(\pi^*) + 2\nu
\end{align*}
It follows that
\[
  \fpdisp_1(\pi') - \fpdisp_{-1}(\pi') = (1 - 2\nu) (\gamma + 2\nu) - 2\nu \leq \gamma
\]
which implies that $\pi'$ is a feasible solution for the fair ERM
problem. This implies that
\[
err(\hat \pi, \cP) \geq err(\pi', \cP) - 2\nu \geq \OPT - 2\nu
\]
This completes the proof.
\end{proof}

Next, we translate the approximation guarantee for the normalized
weighted classification problem to the orginal cost-sensitive
classification problem. This leads to our guarantee stated below.

\begin{lemma}\label{polysize}
  For any $0 < \nu < \gamma/2$, there exists an oracle-efficient
  algorithm that calls CSC oracle over $\H$ at most $O(1/\nu^2)$ times
  and computes a solution $\hat \pi$ that satisfies $\gamma$-EFP and
  has total cost
  \[
    \Ex{h\sim \hat\pi}{ \sum_{j=1}^n (C_j^{1} h(X_j) + C_j^0 (1 -
      h(X_j)))} \leq \OPT_C +\epsilon
  \]
  with $\epsilon = 4\nu \sum_{j=1}^n |C_j^1 - C_j^0|$.
\end{lemma}

The result of Lemma~\ref{polysize} shows a computationally
efficient algorithm that returns an approximate CSC solution with
support size at most $O(1/\nu^2)$. Finally, we will shrink the support
of the solution. To derive a sparse-support solution, we consider a
linear program that computes a probability distribution over the
support of $\hat \pi$. Then we will compute a basic solution obtain
the final sparse solution (e.g. by running a variant of the ellipsoid
algorithm~\cite{Grotschel1981}).

\subsection{Missing Details and Proofs in Section~\ref{feasible}}
\label{feasibleproof}

\begin{algorithm}[H]
\SetAlgoLined
\caption{Coordinate descent algorithm for solving the feasibility program}\label{alg:coor}
\nl \textbf{Input}: history $H_t$ from previous rounds; minimum probability $\mu_t$; target accuracy parameter $\nu$
\nl \textbf{Initialize}: $Q = \mathbf{0}$; Call $\FO(\nu)$ to compute the policy $\pi_0$ that approximately minimizes 
      $\hat L_t(\pi)$ (up to error $\nu$). \\
 \For{$\pi \in \Pi$}{
  \begin{align*}
\mbox{Let}\qquad        \widetilde \Reg_t(\pi) = \max\{\hat L_t(\pi) - \hat L_t(\pi_0), 0\},\qquad
        \tilde b_t(\pi) =   \frac{\widetilde\Reg_t(\pi)}{4(e - 2)\mu_t\ln(T)}
      \end{align*}
	}
 \For{$\pi \in \Pi$}{  
 \begin{align*}
        V_\pi(Q)         &= \Ex{x\sim H_t}{1/Q^{\mu_t}(\pi(x) \mid x)}\\
        S_\pi(Q)         &= \Ex{x\sim H_t}{1/Q^{\mu_t}(\pi(x) \mid x)^2}\\
        \tilde{D}_\pi(Q) &= V_\pi(Q) - (4 + \tilde b_{t-1}(\pi))
      \end{align*}
      }
\nl \If{$\int_{\pi \in \Pi} Q(\pi) (4 + \tilde b_\pi)d\pi > 4$}{
 Replace $Q$ by $c \, Q$ with
        \begin{equation*}
          c = \frac{4}{ \int_\pi Q(\pi) (4 + \tilde{b}_{t-1}(\pi))d\pi} < 1\label{cman}
        \end{equation*}
      }
\nl \If{calling $\FO(\nu)$ for $\pi$ approximating $\max_{\pi'} \tilde{D}_{\pi'}(Q)$, we have $\tilde{D}_{\pi}(Q) > 0$} {
 Add the following (positive) quantity to $Q(\pi)$ while keeping all other weights unchanged:
      \[
        \alpha_\pi(Q) = \frac{V_\pi(Q) + \tilde{D}_\pi(Q)}{2(1 -
          2\mu_t)S_\pi(Q)}
      \]
        }
\nl \Else{Halt. If the sum of the weights $Q$ is smaller than $1$, let $Q$ place the remaining weight on $\pi_0$. Output $Q$ (note the algorithm will draw from $Q^{\mu_t}$).}
\end{algorithm}

\begin{proof}[Proof of lemma~\ref{accuracy}]
  The first oracle call is used to approximately solve
  \begin{align*}
    \arg\min_{\pi} \hat{L}_t(\pi)
      &= \arg\min_{\pi} \frac{1}{t} \sum_{s=1}^t \ell_s \frac{\Pr[\pi(x_s) = a_s]}{Q_s(a_s \mid x_s)}  \\
      &= \frac{1}{\mu_t} \arg\min_{\pi} \sum_{s=1}^t \frac{\mu_t \ell_s}{t} \frac{\Pr[\pi(x_s) = a_s]}{Q_s(a_s \mid x_s)}
  \end{align*}
  where, because $Q_s(a|x)$ is constrained to at least $\mu_s$ (which is decreasing in $s$), the argmin now has weights summing to at most $1$.
  Therefore the oracle, given $\nu$, returns $\tilde{\pi}$ such that
    \[ \min_{\pi} \hat{L}_t(\pi) \leq \hat{L}_t(\tilde{\pi}) \leq \min_{\pi} \hat{L}_t(\pi) + \frac{\nu}{\mu_t} . \]
  This implies that, for all $\pi$,
    \[ \widehat{\Reg}_t(\pi) \geq \widetilde{\Reg}_t(\pi) \geq \widehat{\Reg}_t(\pi) - \frac{\nu}{\mu_t}. \]
  This gives
    \[ b_t(\pi) \geq \tilde{b}_t(\pi) \geq b_t(\pi) - \Lambda_t . \]
  If the first condition is met and the algorithm halts, then $\int Q(\pi)(4 + \tilde{b}_t(\pi))d\pi \leq 4$, implying that the sum of $Q$'s weights is at most $1$ (since $\tilde{b}_t(\pi) \geq 0$), and implying that $\int Q(\pi) (4 + b_t(\pi))d\pi \leq 4 + \Lambda_t$, which is the first inequality.

  Next, the oracle is called once per loop to request
  \begin{align*}
    \arg\max_{\pi} \tilde{D}_{\pi}(Q)
      &= \arg\max_{\pi} \sum_{s=1}^t \frac{1}{t Q_s^{\mu_s}(a_s \mid x_s)} - (4 + \tilde{b}_{t-1}(\pi))
  \end{align*}
  There are two cases, where $\widetilde{\Reg}_t(\pi) = 0$ and otherwise.
  If $0$, then we again obtain an additive $\frac{\nu}{\mu_t}$ approximation.
  Otherwise, after dropping terms that don't depend on $\pi$, we have
  \begin{align*}
    \arg\max_{\pi} \sum_{s=1}^t \frac{1}{t Q_s^{\mu_t}(a_s \mid x_s)} - \frac{\ell_s \Pr[\pi(x_s) = a_s]}{4(e-2)\mu_t t \ln(T) Q_s(a_s \mid x_s)}
  \end{align*}
  Scaling each term by $4(e-2)\ln(T) \mu_t^2$ ensures that the sum of the weights is at most $1$, implying that the approximation we get is again an additive $\Lambda_t$.
  So if $\pi$ is chosen by the algorithm, then $\max_{\pi'} \tilde{D}_{\pi'}(Q) \geq \tilde{D}_{\pi}(Q) \geq \tilde{D}_{\pi'}(Q) - \Lambda_t$.
  Plugging in the guarantee for $b_t$, if we let $D_{\pi}(Q) = V_{\pi}(Q) - (4 + b_{t-1}(\pi))$, then we get
    \[ \max_{\pi^*} D_{\pi^*}(Q) + \Lambda_t \geq \tilde{D}_{\pi}(Q) \geq \max_{\pi^*} D_{\pi^*}(Q) - \Lambda_t . \]
  So if the algorithm halts after obtaining $\pi$ from the oracle with $\tilde{D}_{\pi}(Q) \leq 0$, then $\max_{\pi^*} D_{\pi^*}(Q) \leq \Lambda_t$, which implies the second guarantee.

  To show convergence of the algorithm, consider the following potential
  function
  \[
    \Phi(Q) = \frac{\Ex{H_t}{\RE( \mathcal{U}_2 \| Q^{\mu_t}(\cdot \mid
        x))}}{1 - 2\mu_t} + \frac{\int_{\pi\in\Pi} Q(\pi) \tilde
      b_{t-1}(\pi)d\pi}{4}
  \]
  where $\mathcal{U}_2$ denotes the uniform distribution over the two
  predictions and $\RE(p\| q)$ denotes the unnormalized relative entropy
  between two nonnegative vectors $p$ and $q$ in $\mathbb{R}^2$ (over
  the two predictions):
  \[
    \RE(p \| q) = \sum_{\hat y\in \{\pm 1\}} \left(p_{\hat y}
      \ln(p_{\hat y}/q_{\hat y}) + q_{\hat y} - p_{\hat y} \right).
  \]

  First, we note that any renormalization step does not increase potential, i.e. letting $c = 4 / \int_{\pi}  Q(\pi)(4 + \tilde{b}_t(\pi)d\pi$, if $c < 1$ (which is equivalent to the update condition) then $\Phi(cQ) \leq \Phi(Q)$.
  This is directly proven in Lemma 6 of \cite{minimonster}  and we do not re-prove it. The only difference is that where we used $\tilde{b}_t(\pi)$ in the definition of $c$ and $\Phi$ \cite{minimonster} uses $b_{t-1}(\pi)$; but the proof does not use any property of $b_{t-1}(\pi)$ except nonnegativity.

  Second, we note that a renormalization step can only occur once in a row; after that, either the algorithm halts, or the other condition ($\tilde{D}_{\pi}(Q) > 0$) is triggered.

  Third, when the other condition is triggered, the potential decreases significantly, specifically, by at least $\frac{1}{4(1-2\mu_t)}$.
  This is also directly proven in Lemma 7 of \cite{minimonster}.\footnote{In that paper the potential function is scaled by a factor of $\tau \mu_t$ compared to here, where $\tau > 0$.}
  The only difference is that the proof in that paper uses $\tilde{b}_t(\pi)$ instead of $b_{t-1}(\pi)$, which yields $\tilde{D}_{\pi}(Q)$ rather than $D_{\pi}(Q)$.
  However, the only property of $D_{\pi}(Q)$ used in the proof is $D_{\pi}(Q) > 0$, which is satisfied by $\tilde{D}_{\pi}(Q)$ as well.

  The potential begins with $Q = \mathbf{0}$ at $\Phi(Q) \leq \frac{\ln\frac{1}{\mu_t}}{1-2\mu_t}$, and remains nonnegative by definition, so after a polynomial number of steps, the algorithm satisfies both conditions and halts.
\end{proof}

\subsection{Missing Proofs in Section~\ref{regret}}
\begin{proof}[Proof of Claim~\ref{cla:cover}]
  Let $\pi$ be any policy in $\Pi$. Note, in particular, that
  $\minus \in \Pi$ and let $\pi_{\minus}' \in \Pi_\eta$ such that
\[
  \min_{\pi' \in \Pi_\eta} \|\minus - \pi_{\minus}'|_\infty \leq \eta
\]
Then, we can see that
$$|V(P, \pi, \mu) - V(P, \pi' ,\mu)|_\infty \leq |V(P, \minus, \mu) - V(P, \pi_{\minus}' ,\mu)|_\infty \leq \frac{1}{\mu} - \frac{1}{\mu+\eta} = \frac{\eta}{\mu(\mu+\eta)}$$
$$|\hat V_t(P, \pi, \mu) - \hat V_t(P, \pi' ,\mu)|_\infty \leq |\hat V_t(P, \minus, \mu) - \hat V_t(P, \pi_{\minus}' ,\mu)|_\infty \leq \frac{1}{\mu} - \frac{1}{\mu+\eta} = \frac{\eta}{\mu(\mu+\eta)}$$
\end{proof}

The following lemma follows directly from Lemma 10 of
\cite{minimonster}.

\begin{lemma}[Full version of Lemma~\ref{var_dev}]
 Fix any $\mu \in [0, 1/2]$ and any $\delta \in (0, 1)$. Then, with
 probability $1 - \delta$,
 \[
   V(P, \pi, \mu) \leq 6.4 \hat V_t(P, \pi, \mu) + \frac{75(1 - 2
     \mu)\ln|\Pi_\eta|}{\mu_t^2 t} + \frac{6.3\ln(2|\Pi_\eta|^2
     t^2/\delta)}{\mu_t t} + \frac{2\eta}{\mu_t(\mu_t+\eta)}
 \]
 for all probability distributions $P$ over $\Pi$, all $\pi \in \Pi$,
 and for all $t$. In particular, if
 $$\mu_t \geq \sqrt{\frac{\ln(2|\Pi_{\eta}|t^2/\delta)}{2t}}, t \geq 8\ln(2|\Pi_{\eta}|t^2/\delta)$$
 then,
 $$V(P, \pi, \mu_t) \leq 6.4 \hat V_t(P, \pi, \mu_t) + 162.6 + \frac{2\eta}{\mu_t(\mu_t+\eta)}$$
\end{lemma}

We will make use of the following concentration inequality.

\begin{lemma}[Freedman's inequality~\cite{freedman}]
  Let $Z_1,...,Z_n$ be a martingale difference sequence with
  $Z_i \leq R$ for all $i$. Let
  $V_n = \sum\limits_{i=1}^{n}\Ex{}{Z_i^2 \mid Z_1, \ldots ,
    Z_{i-1}}$. For any $\delta \in (0,1)$ and any
  $\lambda \in [0,1/R]$, with probability at least $1-\delta$
$$\sum\limits_{i=1}^{n} Z_i \leq (e-2)\lambda V_n + \frac{\ln(1/\delta)}{\lambda}$$
\end{lemma}

\begin{proof}[Proof of Lemma~\ref{loss_bound}]
  By applying the Freedman's inequality and union bound, we know that with
  probability $1 - \delta'$, for all $t\in [T]$, $\pi\in\Pi$ and
  $\lambda\in [0, 1/\mu_t]$,
  \begin{align}\label{lai}
    |L(\pi) - \hat L_t(\pi)| \leq (e - 2) \lambda \left(\frac{1}{t}\sum_{s=1}^t
      V(Q_t, \pi, \mu_t) \right) + \frac{\ln(|\Pi_\eta|T/\delta')}{t \lambda}
  \end{align}
  By the result of Lemma~\ref{var_dev}, we know that with probability
  $1 - \delta'$, for all $P\in \Pi$, for any $\mu_t$ and
  $t \geq 8 \ln(2|\Pi_\eta|t^2/\delta')$,
  \begin{align}\label{ce}
    V(P, \pi, \mu) \leq 6.4 \hat V_t(P, \pi, \mu_t) + 162.6 +
    \frac{2\eta}{\mu_t(\mu_t+\eta)}
  \end{align}
  We will condition on events of \eqref{lai} and \eqref{ce} for the remainder of
  the proof, which occurs with probability at least $1 - 2\delta'$.
  Then we can further rewrite
  \begin{align*}
    |L(\pi) - \hat L_t(\pi)| \leq (e - 2)\lambda \left(
    \frac{1}{t} \sum_{s=1}^t \left(6.4 \hat V_t(Q_t, \pi, \mu_t) + 162.6 + \frac{2\eta}{\mu_t(\mu_t + \eta)} \right)\right)  + \frac{\ln(|\Pi_\eta| T/\delta')}{\lambda t}
  \end{align*}
  Recall that by the accuracy guarantee of Lemma~\ref{accuracy}, we
  know for all $\pi\in \Pi$,
  \[
    \hat V_t(Q_t, \pi, \mu_t) \leq 4 + b_{t-1}(\pi) + \Lambda_{t-1}
  \]
  Thus, we can further bound
  \begin{align*}
&\quad    |L(\pi) - \hat L_{t}(\pi)|\\ &\leq (e - 2)\lambda \left(\frac{1}{t} \sum_{s=1}^t \left(6.4 \left(4 + b_{s-1}(\pi) + \Lambda_{s-1} \right) + 162.6 + \frac{2\eta}{\mu_{t}(\mu_{t} + \eta)} \right)\right)  + \frac{\ln(|\Pi_\eta| T/\delta')}{\lambda t}\\
                             &\leq (e - 2)\lambda \left(188.2 + \frac{1}{t} \sum_{s=1}^t \left(6.4 b_{s-1}(\pi)  +  6.4 \Lambda_{s-1} +\frac{2\eta}{\mu_{t}(\mu_{t} + \eta)} \right)\right)  + \frac{\ln(|\Pi_\eta| T/\delta')}{\lambda t}
  \end{align*}
  To complete the proof, we will set $\delta' = \delta/2$.
\end{proof}

\begin{proof}[Proof of Lemma~\ref{bigproof}]
To simplify notation, let
$$C_t = (e - 2)\lambda \left(188.2 + \frac{1}{t} \sum_{s=1}^t \left(6.4
    \Lambda_{s-1} + \frac{\eta}{\mu_{s} (\mu_{s} +
      \eta)}\right)\right) + \frac{\ln(|\Pi_\eta|
  T/\delta)}{t\lambda}$$ Recall that
\[ \Lambda_t := \frac{\nu}{4(e-2)\mu_t^2 \ln(T)} . \] Then as long as
we have $\nu \leq 1/T$ and $\eta \leq 1/T^2$, we have
\[
  C_t \leq 190 (e - 2)\lambda + \frac{\ln(|\Pi_\eta|T/\delta)}{t
    \lambda}
\]
We will prove our result by induction. First, the base case holds
trivially given our choice of $\epsilon_t$. Next, we will show
$\Reg(\pi)\leq 2\widehat \Reg_t(\pi) + \epsilon_t$, and
$\widehat\Reg(\pi)\leq 2 \Reg_t(\pi) + \epsilon_t$ follows
analogously. Observe that for any policy $\pi$, we can first decompose
the regret difference as
\begin{align*}
  \Reg(\pi) - \widehat\Reg_t(\pi) \leq (L(\pi) - \hat L_t(\pi)) - (L(\pi^*) - \hat L_t(\pi^*))
\end{align*}
where $\pi^*$ denotes the optimal policy in $\Pi$. Then using the
result of Lemma~\ref{loss_bound}, we can further bound the regret
difference as follows: for any $\lambda \in [0, \mu_t]$,
\begin{align*}
  &\quad  \Reg(\pi) - \widehat\Reg_t(\pi)\\ &\leq \frac{6.4(e - 2)\lambda}{t} \left(\sum_{s=1}^t b_{s-1}(\pi) + b_{s-1}(\pi^*) \right) + 2 C_t\\
  &= \frac{1.6 \lambda}{\mu_t \ln(T) t} \left(\sum_{s=1}^t \widehat\Reg_{s-1}(\pi) + \widehat\Reg_{s-1}(\pi^*) \right) + 2 C_t\\
  \tag{Induction hypothesis} &\leq \frac{3.2 \lambda}{\mu_t \ln(T) t} \left(\sum_{s=1}^t \Reg(\pi) + \Reg(\pi^*) + \epsilon_{s-1}\right) + 2 C_t\\
  \tag{$\Reg(\pi^*) = 0$}  &\leq \frac{3.2 \lambda}{\mu_t \ln(T) t } \left(t \, \Reg(\pi) + \sum_{s=1}^t  \epsilon_{s-1}\right) + 2 C_t\\
  &\leq \frac{3.2 \lambda}{\mu_t \ln(T)  }  \Reg(\pi)  +  \frac{3.2 \lambda}{\mu_t \ln(T) t }  \left(\sum_{s=1}^t  \epsilon_{s-1}\right) + 2 C_t
\end{align*}
We will set $\lambda = \mu_t/3.2$, which allows us to simplify the
bound
\[
  \Reg(\pi) - \widehat \Reg(\pi) \leq \frac{\Reg(\pi)}{\ln(T)} +
  \frac{1}{\ln(T)t} \left(\sum_{s=1}^t \epsilon_{s-1}\right) + 2 C_t
\]
Since $(1 - 1/\ln(T)) > 1/2$ and
$\mu_t = \frac{3.2\ln(|\Pi_\eta|T/\delta)}{\sqrt{t}}$, it follows that
\begin{align*}
  \Reg(\pi) &\leq 2\widehat \Reg(\pi) + \frac{2}{\ln(T) t}
              \left(\sum_{s=1}^t \epsilon_{s-1}\right) + 4 C_t\\
            &\leq 2\widehat \Reg(\pi) + \frac{2}{\ln(T) t}
              \left(\sum_{s=1}^t \epsilon_{s-1}\right) + 4 \left( \frac{190 (e - 2) \ln(|\Pi_\eta| T/\delta)}{\sqrt{t}} + \frac{1}{\sqrt{t}}\right)\\
            &\leq  2\widehat \Reg(\pi) + \frac{2}{\ln(T) t}
              \left(\sum_{s=1}^t \epsilon_{s-1}\right) + \frac{560 \ln(|\Pi_\eta| T/\delta)}{\sqrt{t}}
\end{align*}
Observe that
$\sum_{s=1}^t \epsilon_{s-1} = 1000(\ln(|\Pi_\eta|T/\delta))
\sum_{s=1}^{t-1} \frac{1}{s} \leq 1000(\ln(|\Pi_\eta|T/\delta))
\sqrt{t}$. This means
\[
  \Reg(\pi) \leq 2\widehat \Reg(\pi) + \frac{2000}{\ln(T) \sqrt{t}}
  (\ln(|\Pi_\eta|T/\delta)) + \frac{560 \ln(|\Pi_\eta|
    T/\delta)}{\sqrt{t}} \leq 2\widehat \Reg(\pi) + \epsilon_t
\]
where the last inequality holds as long as $\ln(T) \geq 5$.
\end{proof}

\begin{proof}[Proof of Theorem~\ref{thm:bandit-regret}]
The cumulative regret of the first
$T_1 = 8\ln(2|\H|^2 T^3/\delta)$ rounds is
trivially bounded by $O(\sqrt{T}\ln(|\H| T/\delta))$. For each of the
remaining rounds, we can use Lemma~\ref{bigproof} to first bound the
per-round regret of the sequence of $Q_t$ as
\[
  \int_{\pi \in \Pi} Q_t(\pi) \Reg(\pi)d\pi \leq 2 \int_{\pi \in \Pi}
  Q_{t-1}(\pi) \widehat\Reg(\pi)d\pi + \epsilon_{t-1}
\]
By the guarantee of Lemma~\ref{accuracy}, we can further bound the
right hand side by
$\left(4(e-2)\mu_{t-1}\ln(T) \right) \Lambda_{t-1} \leq O\left(
  \ln(|\H| T/\delta)/\sqrt{t} \right)$. Summing over rounds, we see
that the cumulative expected regret of the sequence of $Q_t$'s is
bounded by $O\left( \ln(|\H| T/\delta)\sqrt{T} \right)$. Finally, we
need to take into account the $\mu_t$ mixture of uniform prediction
at each round, which incurs an additional cumulative regret of no more
than $O\left( \ln(|\H| T/\delta)\sqrt{T} \right)$.
\end{proof}

\end{document}